\documentclass[10pt,twocolumn,letterpaper]{article}

\usepackage{iccv}

\usepackage{times}
\usepackage{epsfig}
\usepackage{graphicx}
\usepackage{amsmath}
\usepackage{amssymb}
\usepackage{subfigure}
\usepackage{color}
\usepackage{tabularx} 
\usepackage{booktabs} 
\usepackage{array}
\usepackage{multirow}
\usepackage[table]{xcolor}
\usepackage{docmute}
\usepackage{fancyhdr}
\usepackage{amsthm}
\usepackage{color}

\newtheorem{proposition}{Proposition}

\DeclareMathOperator*{\trace}{trace}

\DeclareMathOperator*{\quot}{quot}

\usepackage[pagebackref=true,breaklinks=true,letterpaper=true,colorlinks,bookmarks=false]{hyperref}

\iccvfinalcopy 

\def\iccvPaperID{5442} 

\begin{document}
	
	\title{Minimal Cases for Computing the Generalized Relative Pose using Affine Correspondences
	}
	\author{Banglei Guan$^{1}$,\ \ \ Ji Zhao\thanks{Corresponding author.},\ \ \ Daniel Barath$^{2}$ \ \ and \ \ Friedrich Fraundorfer$^{3,4}$\\
		{\small$^{1}$College of Aerospace Science and Engineering, National University of Defense Technology, China} \\{\small$^{2}$Computer Vision and Geometry Group, Department of Computer Science, ETH Z{\"u}rich} 
		\\{\small$^{3}$Institute for Computer Graphics and Vision, Graz University of Technology, Austria} \\{\small$^{4}$Remote Sensing Technology Institute, German Aerospace Center, Germany}\\
		{\tt\small guanbanglei12@nudt.edu.cn \ zhaoji84@gmail.com \ dbarath@ethz.ch \  fraundorfer@icg.tugraz.at}
	}
	
	\maketitle
	
	\begin{abstract}
		We propose three novel solvers for estimating the relative pose of a multi-camera system from affine correspondences (ACs). A new constraint is derived interpreting the relationship of ACs and the generalized camera model. Using the constraint, we demonstrate efficient solvers for two types of motions assumed. Considering that the cameras undergo planar motion, we propose a minimal solution using a single AC and a solver with two ACs to overcome the degenerate case. Also, we propose a minimal solution using two ACs with known vertical direction, e.g., from an IMU. Since the proposed methods require significantly fewer correspondences than state-of-the-art algorithms, they can be efficiently used within RANSAC for outlier removal and initial motion estimation. The solvers are tested both on synthetic data and on real-world scenes from the KITTI odometry benchmark. It is shown that the accuracy of the estimated poses is superior to the state-of-the-art techniques.
	\end{abstract}
	
	\vspace{-20pt}
	\section{Introduction}
	Relative pose estimation from two views of a camera, or a multi-camera system is regarded as a fundamental problem in computer vision~\cite{HartleyZisserman-472,clipp2008robust,scaramuzza2011visual,schoenberger2016sfm,Zhao2020GEM}, which plays an important role in simultaneous localization and mapping (SLAM) and structure-from-motion (SfM). Thus, improving the accuracy, efficiency and robustness of relative pose estimation algorithms is always an important research topic~\cite{hee2014relative,ventura2015efficient,Agarwal2017,guan2018visual,barath2020making,Eichhardt2020Relative,Li2020Relative}. Motivated by the fact that multi-camera systems are available in self-driving cars, micro aerial vehicles or AR headsets, this paper investigates the problem of estimating the relative pose of multi-camera systems from affine correspondences (ACs), see Fig.~\ref{fig:AffineTransformation}.
	\begin{figure}[t]
		\begin{center}
			\includegraphics[width=0.95\linewidth]{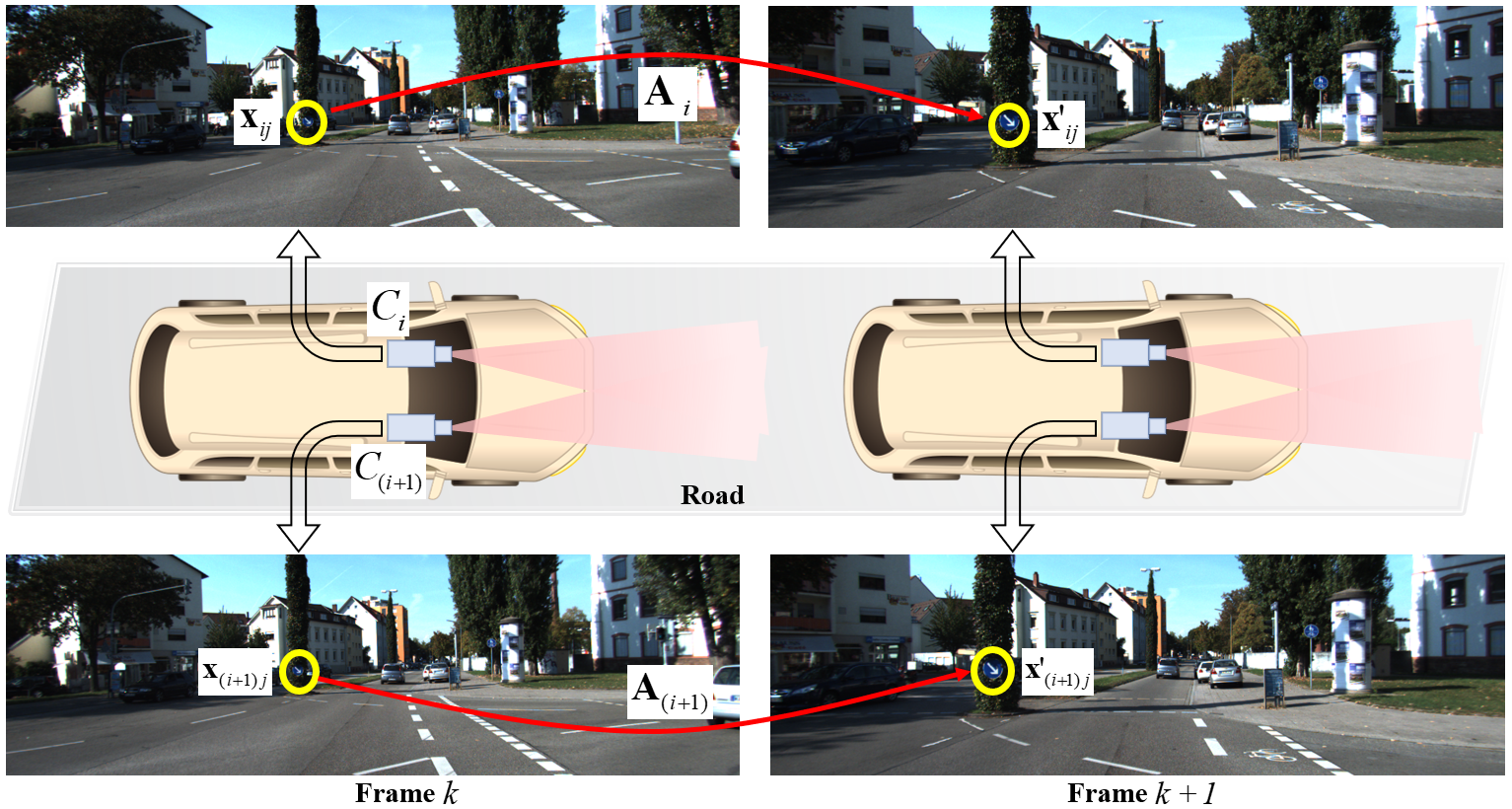}
		\end{center}
		\vspace{-3pt}
		\caption{An affine correspondence in camera $C_i$ between consecutive frames $k$ and $k+1$. The local affine transformation $\mathbf{A}$ relates the infinitesimal patches around point correspondence (${\mathbf{x}}_{ij}$, ${\mathbf{x}}'_{ij}$).}
		\vspace{-8pt}
		\label{fig:AffineTransformation}
	\end{figure}
	
	Since a multi-camera system contains multiple individual cameras connected by being fixed to a single rigid body, it has the advantage of large field-of-view and high accuracy~\cite{Sweeney_2015_ISMAR,Fragoso_2020_CVPR}. The main difference of a multi-camera system and a standard pinhole camera is the absence of a single projection center~\cite{pless2003using}. Due to the different camera model, the relative pose estimation problem of multi-camera systems~\cite{henrikstewenius2005solutions} is different from the monocular cameras~\cite{nister2004efficient,Guan2020CVPR}, which  results in different equations.
	In order to remove outlier matches, most of the state-of-the-art SLAM and SfM pipelines using a multi-camera system~\cite{hane20173d,heng2019project} apply the relative pose estimation algorithms repeatedly in a robust estimation framework, \emph{e.g.} the Random Sample Consensus (RANSAC)~\cite{fischler1981random}. This outlier removal process has to be efficient, which directly affects the real-time performance of SLAM and SfM. The computational complexity and, thus, the processing time of the RANSAC procedure depends exponentially on the number of points required for the relative pose estimation of multi-camera system.  
	
	Therefore, exploring the minimal solutions for relative pose estimation of multi-camera system is of significant importance and has received sustained attention~\cite{henrikstewenius2005solutions,li2008linear,clipp2008robust,kim2009motion,lim2010estimating,ventura2015efficient,kneip2016generalized}. The idea of deriving minimal solutions for relative pose estimation of multi-camera systems ranges back to the work of Stew{\'e}nius \emph{et al.} with the 6-point method~\cite{henrikstewenius2005solutions}. Then other classical works have been subsequently proposed, such as the 17-point linear method~\cite{li2008linear} and techniques based on iterative optimization~\cite{kneip2014efficient}. 
	The minimal number of necessary points can be further reduced by taking additional motion constraints into account~\cite{hee2013motion} or exploiting the measurements from other sensors, like an inertial measurement unit (IMU)~\cite{hee2014relative,sweeney2014solving,sweeney2015computing,liu2017robust,Martyushev2020Efficient}. Typically, the assumption of planar motion or considering known vertical direction are common for self-driving cars and ground robots~\cite{choi2018fast,hajder2019relative,Guan2020CVPR,SaurerVasseur-457,Li2020Relative}, which makes the outlier removal more efficient and numerically more stable.

	All previously mentioned relative pose solvers estimate the pose parameters from a set of point correspondences (PCs), \emph{e.g.}, coming from SIFT~\cite{Lowe2004Distinctive} or SURF~\cite{Bay2008346} detectors. Due to containing more information about the underlying surface geometry than PCs, ACs enable to estimate the pose from fewer correspondences. In this paper, we focus on the relative pose estimation of a multi-camera system from ACs, instead of PCs. The contributions of this paper are:
	\begin{itemize}
		\vspace{-3pt}
		\item A new constraint that interprets the relationship of ACs and the generalized camera model is derived under general motion. This constraint can be easily generalized to special cases of multi-camera motion, \emph{e.g.}, planar motion and known vertical direction.
		\vspace{-3pt}
		\item When the motion is planar (\emph{i.e.}, the body to which the cameras are fixed moves on a plane; 3DOF), a single AC is sufficient to recover the planar motion of a multi-camera system. In order to deal with the degenerate case of the 1AC solver, we also propose a new method to estimate the relative pose from two ACs. The point-based solver~\cite{hee2013motion} requires at least two PCs and requires the Ackermann motion model to hold.		
		\vspace{-3pt}
		\item A third solver is proposed for the case when the vertical direction is known (4DOF), \emph{e.g.}, from an IMU attached to the multi-camera system. We show that two ACs are enough to recover the relative pose. In contrast, the point-based solver requires four PCs~\cite{hee2014relative,sweeney2014solving}. 
	\end{itemize}
	
	\section{\label{sec:relatedwork}Related Work}
	Due to the absence of a single center of projection, the camera model of multi-camera systems is different from the standard pinhole camera. Pless proposed to express the light rays as Pl\"{u}cker lines and derived the generalized camera model which has become a standard representation for the multi-camera systems~\cite{pless2003using}. Stew{\'e}nius~\emph{et al.} proposed the first minimal solution to estimate the relative pose of a multi-camera system from 6 PCs, which produces up to 64 solutions~\cite{henrikstewenius2005solutions}. Li~\emph{et al.} provided several linear solvers to compute the relative pose, among which the most commonly used one requires 17 PCs~\cite{li2008linear}. Kneip~\emph{et al.} proposed an iterative approach for the relative pose estimation based on eigenvalue minimization~\cite{kneip2014efficient}. Ventura~\emph{et al.} used first-order approximation of the rotation to simplify the problem and estimated the relative pose from 6 PCs~\cite{ventura2015efficient}.        
	
	By considering additional motion constraints or using additional information provided by an IMU, the number of required PCs can be further reduced. Lee~\emph{et al.} presented a minimal solution with two PCs for the ego-motion estimation of a multi-camera system, which constrains the relative motion by the Ackermann motion model~\cite{hee2013motion}. In addition, a variety of algorithms have been proposed when a common direction of the multi-camera system is known,~\emph{i.e.} an IMU provides the roll and pitch angles of the multi-camera system. The relative pose estimation with known vertical direction requires a minimum of 4 PCs~\cite{hee2014relative,sweeney2014solving,liu2017robust}.
	
	Exploiting the additional affine parameters besides the image coordinates has been recently proposed for the relative pose estimation of monocular cameras, which reduces the number of required points significantly. Bentolila~\emph{et al.} estimated the fundamental matrix from three ACs~\cite{bentolila2014conic}. Raposo~\emph{et al.} computed homography and essential matrix using two ACs~\cite{raposo2016theory}. Barath~\emph{et al.} derived the constraints between the local affine transformation and the essential matrix and recovered the essential matrix from two ACs~\cite{barath2018efficient}. Hajder~\emph{et al.}~\cite{hajder2019relative} and Guan~\emph{et al.}~\cite{Guan2020CVPR, Guan_TCYB2021} proposed several minimal solutions for relative pose from a single AC under the planar motion assumption or with knowledge of a vertical direction. The above mentioned works are only suitable for the monocular perspective camera. For multi-camera systems, Alyousefi~\emph{et al.} recently proposed a linear solver to estimate the relative pose using 6 ACs~\cite{alyousefi2020multi}. Guan~\emph{et al.} estimated the relative pose from 2 ACs by utilizing a first-order rotation approximation~\cite{Guan_ICRA2021}. In this paper, we focus on the minimal number of ACs to estimate the relative pose of a multi-camera system. Table~\ref{tab:summary_solver} shows a summary of the solvers, including the DOF of the motion, feature types and number of points required.
	\vspace{-5pt} 
	\begin{table}[htbp]
		\caption{Relative pose solvers for multi-camera systems.}
		\begin{center}
			\setlength{\tabcolsep}{1.0mm}{
				\scalebox{0.70}{
					\begin{tabular}{|c||c|c|c|c|c|c|c|c|c|c|}
						\hline
						\small{Solver} &  \small{\cite{li2008linear}}  & \small{\cite{kneip2014efficient}} & \small{\cite{henrikstewenius2005solutions}} &  \small{\cite{alyousefi2020multi}}&  \small{\cite{hee2014relative}}  & {\small{\cite{sweeney2014solving}}}   & {\small{\cite{liu2017robust}}}   & \small{\textbf{1AC~plane}}   & \small{\textbf{2AC~plane}} & \small{\textbf{2AC~vertical}} \\
						\hline
						\small{Motion} & \multicolumn{4}{c|}{\small{6DOF}} & \multicolumn{3}{c|}{\small{4DOF}} & \multicolumn{2}{c|}{\small{3DOF}} & \multicolumn{1}{c|}{\small{4DOF}}\\ 		
						\hline
						\small{Feature}& \multicolumn{3}{c|}{\small{PCs}} & \multicolumn{1}{c|}{\small{ACs}} & \multicolumn{3}{c|}{\small{PCs}} & \multicolumn{3}{c|}{\small{ACs}}\\ 						
						\hline
						\small{Point \#}& \multicolumn{1}{c|}{\small{17}} & \multicolumn{1}{c|}{\small{8}} & \multicolumn{2}{c|}{\small{6}} & \multicolumn{3}{c|}{\small{4}} & \multicolumn{1}{c|}{\small{1}} & \multicolumn{2}{c|}{\small{2}}\\ 			
						\hline
			\end{tabular}}}
		\end{center}
		\label{tab:summary_solver}
	\vspace{-30pt} 
	\end{table}	
	
	\section{\label{sec:6DOFmotion}Geometric Constraints from ACs}
	A multi-camera system is made up of individual cameras denoted by $C_i$, as shown in Fig.~\ref{fig:AffineTransformation}. Take an AC seen by same camera for an example. The geometric constraints can be easily generalized to the case that the AC is seen by different cameras. The extrinsic parameters of camera $C_i$ expressed in a multi-camera reference frame are represented as $(\mathbf{R}_i,\mathbf{t}_i)$. For general motion, there is a 3DOF rotation and a 3DOF translation between two reference frames at time $k$ and $k+1$. Rotation $\mathbf{R}$ using Cayley parameterization and translation $\mathbf{t}$ can be written as: 
	\begin{equation}
	\small
	\begin{aligned}	
	& \mathbf{R} = \frac{1}{1+q_x^2+q_y^2+q_z^2} \ . \\ 
	& \begin{bmatrix}{1+q_x^2-q_y^2-q_z^2}&{2{q_x}{q_y}-2{q_z}}&{2{q_y}+2{q_x}{q_z}}\\
	{2{q_x}{q_y}+2{q_z}}&{1-q_x^2+q_y^2-q_z^2}&{2{q_y}{q_z}-2{q_x}}\\
	{2{q_x}{q_z}-2{q_y}}&{2{q_x}+2{q_y}{q_z}}&{1-q_x^2-q_y^2+q_z^2}
	\end{bmatrix},\\
	\end{aligned}		
	\label{eq:R6dof1}
	\end{equation}
	\begin{equation} \\
	\mathbf{t} = \begin{bmatrix}
	{t_x}& \
	{t_y}& \
	{t_z}
	\end{bmatrix}^T,
	\label{eq:T6dof1}
	\end{equation} 
	where $[1,q_x,q_y,q_z]^T$ is a homogeneous quaternion vector. Note that 180 degree rotations are prohibited in Cayley parameterization, but this is a rare case for consecutive frames.  
	
	\subsection{Generalized Camera Model}
	We give a brief description of generalized camera model (GCM)~\cite{pless2003using}. Let us denote an AC in camera $C_i$ between consecutive frames $k$ and $k+1$ as $({\mathbf{x}}_{ij}, {\mathbf{x}}'_{ij}, \mathbf{A})$, where ${\mathbf{x}}_{ij}$ and ${\mathbf{x}}'_{ij}$ are the normalized homogeneous image coordinates of feature point $j$ and $\mathbf{A}$ is a 2$\times$2 local affine transformation. Indices $i$ and $j$ are the camera and point index, respectively. The related local affine transformation $\mathbf{A}$ is a 2$\times$2 linear transformation which relates the infinitesimal patches around ${\mathbf{x}}_{ij}$ and ${\mathbf{x}}'_{ij}$~\cite{barath2018five}. The normalized homogeneous image coordinates $({\mathbf{p}}_{ij}, {\mathbf{p}}'_{ij})$ expressed in the multi-camera reference frame are given as
	\begin{equation}
	{\mathbf{p}}_{ij} = {\mathbf{R}_i}{\mathbf{x}}_{ij},\qquad
	{\mathbf{p}}'_{ij} = {\mathbf{R}_i}{\mathbf{x}}'_{ij}.
	\label{eq:imagecoord6dof}
	\end{equation}
	The unit direction of rays $({\mathbf{u}}_{ij}, {\mathbf{u}}'_{ij})$ expressed in the multi-camera reference frame are given as: ${\mathbf{u}}_{ij} = {\mathbf{p}}_{ij}/{{\|}{{\mathbf{p}}_{ij}}{\|}}$, ${\mathbf{u}}'_{ij} = {\mathbf{p}}'_{ij}/{{\|}{{\mathbf{p}}'_{ij}}{\|}}$. The 6-dimensional vector Pl\"{u}cker lines corresponding to the rays are denoted as ${\mathbf{l}}_{ij} = [{\mathbf{u}}_{ij}^T, \ ({\mathbf{t}}_i\times {\mathbf{u}}_{ij})^T]^T$, ${\mathbf{l}}'_{ij} = [{{\mathbf{u}}'_{ij}}^T, \ ({\mathbf{t}}_i\times {\mathbf{u}}'_{ij})^T]^T$. The generalized epipolar constraint is written as~\cite{pless2003using}
	\begin{equation} 
	{{\mathbf{l}}'^T_{ij}}
	\begin{bmatrix} {{{\left[ {\mathbf{t}} \right]}_ \times }{\mathbf{R}}}, & {\mathbf{R}} \\ {\mathbf{R}},  & {\mathbf{0}} \end{bmatrix}
	{{\mathbf{l}}_{ij}} = 0,
	\label{GECS6dof}
	\end{equation}
	where ${{\mathbf{l}}'^T_{ij}}$ and ${{\mathbf{l}}_{ij}}$ are Pl\"{u}cker lines between two consecutive frames at time $k$ and $k+1$. 
	
	\subsection{Affine Transformation Constraint}
	We denote the transition matrix of camera coordinate system $C_i$ between consecutive frames $k$ and $k+1$ as $(\mathbf{R}_{Ci},\mathbf{t}_{Ci})$, which is represented as:
	{ \begin{equation}
		\begin{aligned}
		&\begin{bmatrix}
		{\mathbf{R}_{Ci}}&{\mathbf{t}_{Ci}}\\
		{{\mathbf{0}}}&{1}\\
		\end{bmatrix} = \begin{bmatrix}{\mathbf{R}_{i}}&{\mathbf{t}_{i}}\\
		{{\mathbf{0}}}&{1}\\
		\end{bmatrix}^{-1}\begin{bmatrix}{\mathbf{R}}&{\mathbf{t}}\\
		{{\mathbf{0}}}&{1}\\
		\end{bmatrix}\begin{bmatrix}{\mathbf{R}_{i}}&{\mathbf{t}_{i}}\\
		{{\mathbf{0}}}&{1}\\
		\end{bmatrix} \\
		& \qquad \ \ \ =\begin{bmatrix}{{\mathbf{R}_{i}^T}{\mathbf{R}}{\mathbf{R}_{i}}}& \ {{\mathbf{R}_{i}^T}{\mathbf{R}}{\mathbf{t}_{i}}+{\mathbf{R}_{i}^T}{\mathbf{t}}-{\mathbf{R}_{i}^T}{\mathbf{t}_{i}}}\\
		{{\mathbf{0}}}& \ {1}\\
		\end{bmatrix}.
		\end{aligned}
		\label{eq:transformationmatrix6dof}
		\end{equation}}\\
	Essential matrix $\mathbf{E}$ of two frames of camera $C_i$ is given as:
	\begin{equation}
	\begin{aligned}
	\mathbf{E} = [\mathbf{t}_{Ci}]_{\times}\mathbf{R}_{Ci}
	= {\mathbf{R}_{i}^T}[{\mathbf{R}_{i}}\mathbf{t}_{Ci}]_{\times}{{\mathbf{R}}{\mathbf{R}_{i}}},
	\end{aligned}
	\label{eq:E6dof}
	\end{equation}
	where $\left[{\mathbf{R}_{i}}\mathbf{t}_{Ci}\right]_{\times}={\mathbf{R}}[{\mathbf{t}_{i}}]_{\times}{{\mathbf{R}}^T} + [{\mathbf{t}}]_{\times} - [{\mathbf{t}_{i}}]_{\times}$. The relationship of essential matrix $\mathbf{E}$ and local affine transformation $\mathbf{A}$ is formulated as follows~\cite{barath2018efficient}:
	\vspace{-1pt} 
	\begin{equation}
	(\mathbf{E}^{T}{\mathbf{x}}'_{ij})_{(1:2)} = -(\hat{\mathbf{A}}^{T}\mathbf{E}{\mathbf{x}}_{ij})_{(1:2)},
	\vspace{-1pt} 	
	\label{eq:E6dof_Ac1}
	\end{equation}
	where $\mathbf{n}_{ij}\triangleq{\mathbf{E}^{T}{\mathbf{x}}'_{ij}}$ and $\mathbf{n}'_{ij}\triangleq{\mathbf{E}{\mathbf{x}}_{ij}}$ denote the epipolar lines in their implicit form in frames of camera $C_i$ at times $k$ and $k+1$. The subscript 1 and 2 represent the first and second equations of the equation system, respectively. $\hat{\mathbf{A}}$ is a $3\times3$ matrix: $\hat{\mathbf{A}} = [\mathbf{A} \ \mathbf{0}; \mathbf{0} \ 0]$. By substituting Eq.~\eqref{eq:E6dof} into Eq.~\eqref{eq:E6dof_Ac1}, we obtain:
	\vspace{-1pt} 
	\begin{eqnarray}
	\begin{aligned}		
	({\mathbf{R}_{i}^T}{\mathbf{R}^T}&{[{\mathbf{R}_{i}}\mathbf{t}_{Ci}]_{\times}^T}{\mathbf{R}_{i}}{\mathbf{x}}'_{ij})_{(1:2)} \\
	&= -(\hat{\mathbf{A}}^{T}{\mathbf{R}_{i}^T}[{\mathbf{R}_{i}}\mathbf{t}_{Ci}]_{\times}{{\mathbf{R}}{\mathbf{R}_{i}}}{\mathbf{x}}_{ij})_{(1:2)}.
	\end{aligned}
	\vspace{-1pt} 
	\label{eq:E6dof_Ac2}
	\end{eqnarray}
	Based on Eq.~\eqref{eq:imagecoord6dof}, the above equation is reformulated and expanded as follows:
	\vspace{-1pt} 
	\begin{equation}
	\begin{aligned}		
	({\mathbf{R}_{i}^T}&([{\mathbf{t}_{i}}]_{\times}{\mathbf{R}}^T  +   {\mathbf{R}^T}[{\mathbf{t}}]_{\times} -  {\mathbf{R}^T}[{\mathbf{t}_{i}}]_{\times}){\mathbf{p}}'_{ij})_{(1:2)} = \\
	&(\hat{\mathbf{A}}^{T}{\mathbf{R}_{i}^T}({\mathbf{R}}[{\mathbf{t}_{i}}]_{\times} + [{\mathbf{t}}]_{\times}{\mathbf{R}} - [{\mathbf{t}_{i}}]_{\times}{\mathbf{R}}){\mathbf{p}}_{ij})_{(1:2)}.
	\end{aligned}
	\vspace{-1pt} 
	\label{eq:E6dof_Ac6}
	\end{equation}
	
	Eq.~\eqref{eq:E6dof_Ac6} interprets the new epipolar constraints which a local affine transformation implies on the $i$-th camera from a multi-camera system between two frames $k$ and $k+1$.
	
	For each AC $({\mathbf{x}}_{ij}, {\mathbf{x}}'_{ij}, \mathbf{A})$, we get three polynomials based on Eqs.~\eqref{GECS6dof} and~\eqref{eq:E6dof_Ac6}, see supplementary material for details. Motivated by scenarios like self-driving cars, ground robots or AR headsets, we investigate relevant special cases of multi-camera motion, \emph{i.e.}, planar motion and motion with known vertical direction, see Fig.~\ref{fig:Specialcases}. We show that two special cases can be efficiently solved with ACs. 
	
	\section{\label{sec:planarmotion}Relative Pose under Planar Motion}
	\vspace{-2pt} 
	\begin{figure}[ht]
		\begin{center}	
			\subfigure[Planar motion]
			{
				\includegraphics[width=0.31\linewidth]{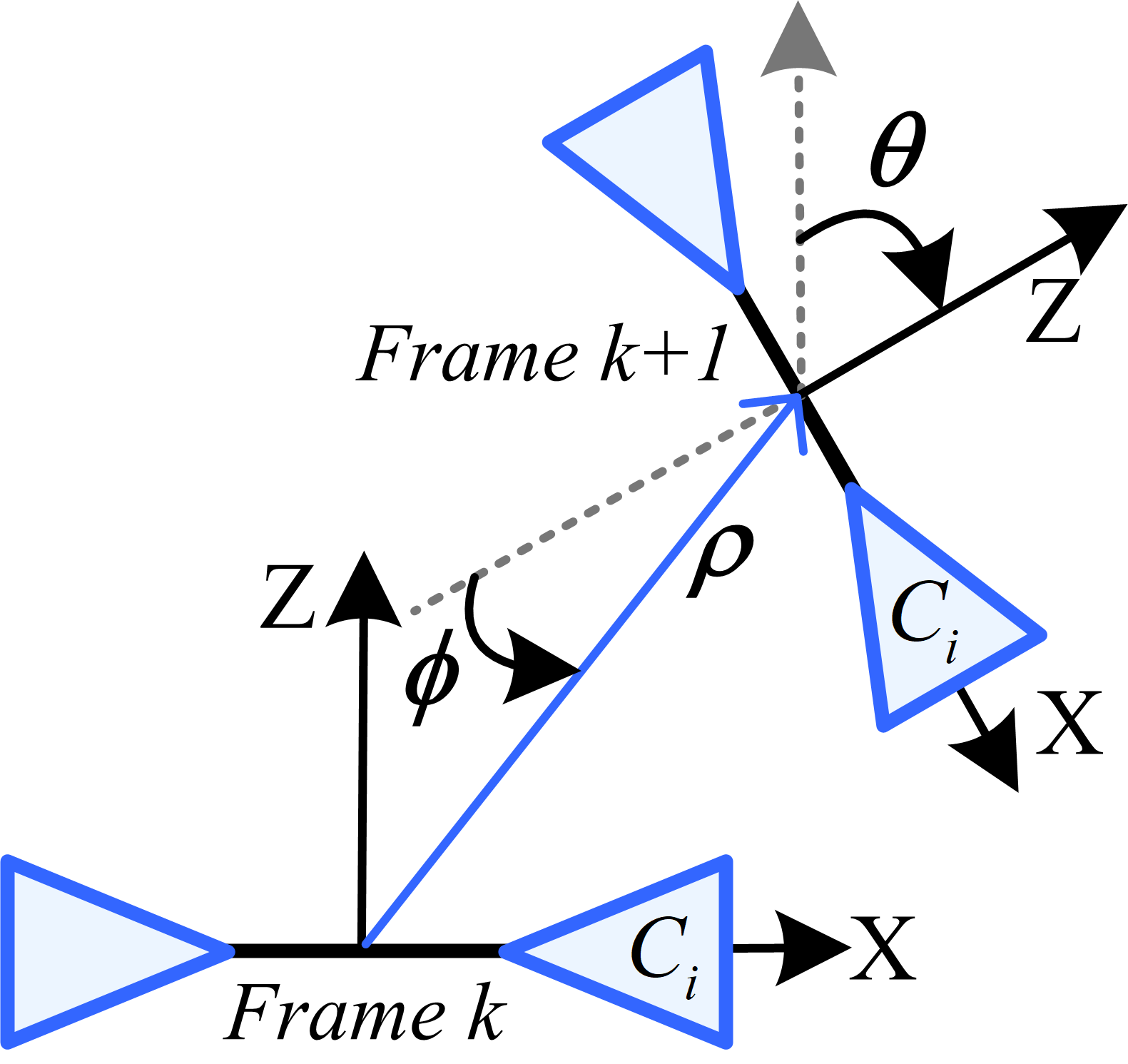}
			}
			\hspace{0.2in}
			\subfigure[Motion with known vertical direction]
			{
				\includegraphics[width=0.55\linewidth]{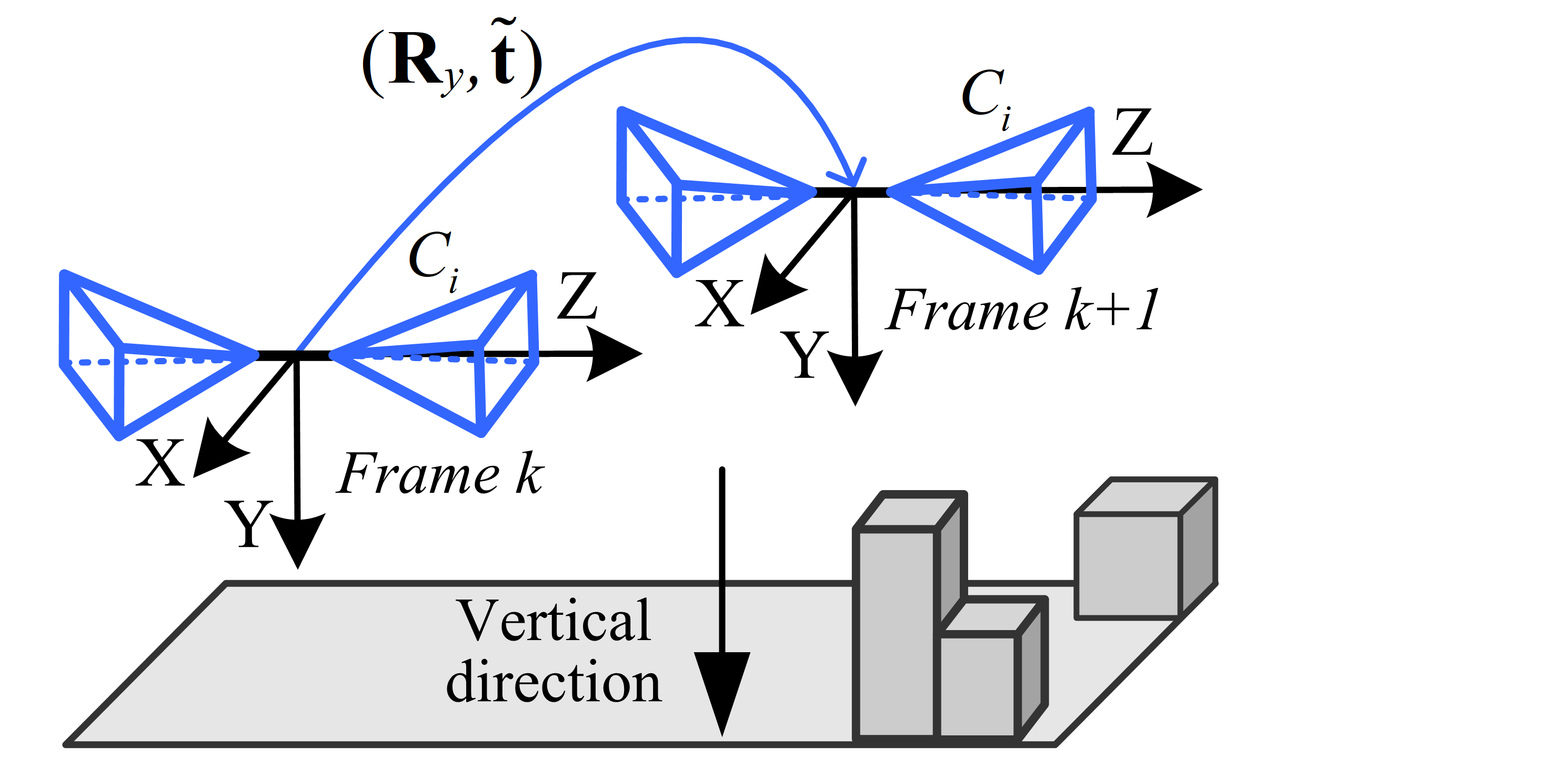}
			}
		\end{center}
		\vspace{-5pt} 
		\caption{Special cases of multi-camera motion: (a) Planar motion in top-view. There are three unknowns: yaw angle $\theta$, translation direction $\phi$ and translation distance $\rho$. (b) Motion with known vertical direction. There are four unknowns: a Y-axis rotation $\mathbf{R}_{y}$ and 3D translation $\tilde{\mathbf{t}} =[{\tilde{t}_x}, {\tilde{t}_y}, {\tilde{t}_z}]^T$.}
		\label{fig:Specialcases}
	\end{figure} 
	
	When assuming that the body, to which the camera system is rigidly fixed, moves on a planar surface (as visualized in Fig.~\ref{fig:Specialcases}(a)), there are only a Y-axis rotation and 2D translation between the reference frames $k$ and $k+1$. Similar to Eqs.~\eqref{eq:R6dof1} and~\eqref{eq:T6dof1}, the rotation $\mathbf{R}=\mathbf{R}_{y}$ and the translation $\mathbf{t}$ from frame $k$ to $k+1$ is written as: 
	\begin{equation}
	\begin{aligned}	
	\mathbf{R}_{y} & = \frac{1}{1+{q_y^2}}\begin{bmatrix}{1-{q_y^2}}&0&{-2{q_y}}\\
	0&1+{q_y^2}&0\\
	{2{q_y}}&0&{1-{q_y^2}}
	\end{bmatrix}, \\
	\mathbf{t} & = \begin{bmatrix}
	{t_x}& \
	{0}& \
	{t_z}
	\end{bmatrix}^T.
	\end{aligned}	
	\label{eq:Ryt1}
	\end{equation}
	where ${q_y}=\tan(\frac{\theta}{2})$, $t_x={\rho\sin{(\phi)}}$, $t_z={-\rho\cos{(\phi)}}$, $\rho$ is the distance between two multi-camera reference frames.  
	
	\subsection{Solution by Reduction to a Single Polynomial}
	By substituting Eq.~\eqref{eq:Ryt1} into Eqs.~\eqref{GECS6dof} and~\eqref{eq:E6dof_Ac6}, we get an equation system of three polynomials for three unknowns $q_y$, $t_x$ and $t_z$.
	Since an AC generally provides three independent constraints for relative pose, a single AC is sufficient to recover the planar motion of a multi-camera system. After separating $q_y$ from $t_x$, $t_z$, the three independent constraints from an AC form matrix equation:
	\begin{equation} 
	\frac{1}{1+{q_y^2}}\underbrace {\begin{bmatrix}
		{M_{11}}& {M_{12}}& {M_{13}}\\
		{M_{21}}& {M_{22}}& {M_{23}}\\
		{M_{31}}& {M_{32}}& {M_{33}}
		\end{bmatrix}}_{{\mathbf{M}}\left( {{q_y}} \right)}
	\begin{bmatrix}
	{{{t}_x}}\\
	{{{t}_z}}\\
	1
	\end{bmatrix} = {\mathbf{0}},
	\label{eq:euq_q1}
	\end{equation}
	where $M_{ij}$ $(i,j \in [1, 3])$ is an element of coefficient matrix ${\mathbf{M}(q_y)}$ and are formed by the polynomial coefficients and one unknown variable $q_y$, see supplementary material for details. Since ${\mathbf{M}(q_y)}$ is a square matrix, Eq.~\eqref{eq:euq_q1} has a non-trivial solution only if the determinant of ${\mathbf{M}(q_y)/(1+{q_y^2})}$ is zero. The expansion of $\det({\mathbf{M}(q_y)}/(1+{q_y^2}))=0$ gives a 4-degree univariate polynomial as follows:
	\begin{eqnarray}
	\begin{aligned}		
	\quot(\textstyle \sum_{i=0}^6 w_i q_y^i, {q_y^2}+1) = 0,
	\end{aligned}
	\label{eq:euq_q2}
	\end{eqnarray}
	where $\quot(a, b)$ means calculating the quotient of $a$ divided by $b$, $w_{0},\ldots,w_{6}$ are formed by a Pl\"{u}cker line correspondence and an affine transformation between the corresponding feature points. 
	
	Note that the coefficients are divided by ${q_y^2}+1$, which reduces the polynomial degree and improves the efficiency of the solution. The univariate polynomial Eq.~\eqref{eq:euq_q2} leads to an explicit analytic solution with a maximum of 4 real roots. Once the solutions for $q_y$ are found, the remaining unknowns $t_x$ and $t_z$ are solved by substituting $q_y$ into ${\mathbf{M}(q_y)}$ and solving the linear system via calculating its null vector. Finally, the rotation matrix $\mathbf{R}_{y}$ is recovered from Eq.~\eqref{eq:Ryt1}. 
	
	However, we proved that the solver using a single AC has a degenerate case, \emph{i.e.}, the distances between the motion plane and optical centers of the cameras being equal, see supplementary material for details. This degenerate case might happen in the self-driving scenario, which would lead to that both the translation direction and the translation scale cannot be calculated using one AC. To overcome this issue, two ACs are used to estimate the relative pose. For example, the first and second constraints of the first AC, and the first constraint of the second AC are also stacked into three equations in three unknowns, just as Eq.~\eqref{eq:euq_q1}. The solution procedure remains the same, except that the code for constructing the coefficient matrix ${\mathbf{M}(q_y)}$ is replaced.      
	
	\section{\label{sec:knownverticaldirection}Relative Pose with Known Vertical Direction}
	In this section a minimal solution using two ACs is proposed for relative motion estimation for multi-camera systems with known vertical direction, see Fig.~\ref{fig:Specialcases}(b). In this case, an IMU is coupled with the multi-camera system and the relative rotation between the IMU and the reference frame is known. The IMU provides the known roll and pitch angles for the reference frame. So the reference frame can be aligned with the measured vertical direction, such that the X-Z-plane of the aligned reference frame is parallel to the ground plane and the Y-axis is parallel to the vertical direction. Rotation $\mathbf{R}_{\text{\text{imu}}}$ for aligning the reference frame to the aligned reference frame is written as:  
	\begin{equation}
	\begin{aligned}
	&\mathbf{R}_{\text{\text{imu}}} = \mathbf{R}_{p}\mathbf{R}_{r} \\
	&= \begin{bmatrix}1&0&0\\
	0&\cos(\theta_p)&{\sin(\theta_p)}\\
	0&{-\sin(\theta_p)}&{\cos(\theta_p)}
	\end{bmatrix}\begin{bmatrix}
	{\cos(\theta_r)}&{\sin(\theta_r)}&0\\
	{ -\sin(\theta_r)}&{\cos(\theta_r)}&0\\
	0&0&1
	\end{bmatrix}, \nonumber
	\end{aligned}
	\label{eq:RxRz}
	\end{equation}
	where $\theta_r$ and $\theta_p$ are roll and pitch angles provided by the coupled IMU, respectively. Thus, there are only a Y-axis rotation $\mathbf{R}=\mathbf{R}_{y}$ and 3D translation $\tilde{\mathbf{t}}= {{{\mathbf{R}}}'_{\text{imu}}}{\mathbf{t}} =[{\tilde{t}_x}, {\tilde{t}_y}, {\tilde{t}_z}]^T$ to be estimated between the aligned multi-camera reference frames at time $k$ and $k+1$. In this section, we show that the geometric constraints in Section~\ref{sec:6DOFmotion} can be generalized to the multi-camera motion with known vertical direction.   
	
	\subsection{Generalized Camera Model}
	Let us denote the rotation matrices from the roll and pitch angles of the two corresponding multi-camera reference frames
	at time $k$ and $k+1$ as $\mathbf{R}_{\text{\text{imu}}}$ and $\mathbf{R}'_{\text{\text{imu}}}$. The relative rotation between two reference frames is 
	\begin{equation} 
	{\mathbf{R}} = {(\mathbf{R}'_{\text{\text{imu}}})^T}{\mathbf{R}_{y}}{\mathbf{R}_{\text{\text{imu}}}}.
	\label{eq:Rv}
	\end{equation}
	We substitute Eq.~\eqref{eq:Rv} into Eq.~\eqref{GECS6dof} yields:
	{{\begin{equation}
			\begin{aligned}
			{\underbrace {\left(\begin{bmatrix}
					{{{{\mathbf{R}}}'_{\text{\text{imu}}}}}& {\mathbf{0}}\\
					{{\mathbf{0}}}& {{{{\mathbf{R}}}'_{\text{\text{imu}}}}}\\
					\end{bmatrix}{\mathbf{l}'_{ij}} \right)^T}_{\tilde{\mathbf{l}}'_{ij}}}
			&\begin{bmatrix}{{{\left[ {{\tilde{\mathbf{t}}}} \right]}_ \times } {{\mathbf{R}}_y}}&{{{\mathbf{R}}_y}}\\
			{{{\mathbf{R}}_y}}&{\mathbf{0}}
			\end{bmatrix}\\
			\vspace{-5pt} 
			&{\underbrace {\left(\begin{bmatrix}
					{{{\mathbf{R}}_{\text{\text{imu}}}}}& {\mathbf{0}}\\
					{{\mathbf{0}}}& {{{\mathbf{R}}_{\text{\text{imu}}}}}\\
					\end{bmatrix}{\mathbf{l}_{ij}} \right)}_{\tilde{\mathbf{l}}_{ij}}}= 0,
			\end{aligned}
			\label{eq:GECSIMU}
			\end{equation}}}\\
	\noindent 
	where ${\tilde{\mathbf{l}}_{ij}} \leftrightarrow {\tilde{\mathbf{l}}'_{ij}}$ are the corresponding Pl\"{u}cker lines expressed in the aligned multi-camera reference frame.
	
	\subsection{Affine Transformation Constraint}
	In this case, the transition matrix of the camera coordinate system $C_i$ between consecutive frames $k$ and $k+1$ is represented as follows:
	{\begin{equation}
		\begin{aligned}
		&\begin{bmatrix}
		{\mathbf{R}_{Ci}}&\ {\mathbf{t}_{Ci}}\\
		{{\mathbf{0}}}&\ {1}\\
		\end{bmatrix} 
		= \left(\begin{bmatrix}{\mathbf{R}'_{\text{imu}}}&{\mathbf{0}}\\
		{{\mathbf{0}}}&{1}\\
		\end{bmatrix}\begin{bmatrix}{\mathbf{R}_{i}}&{\mathbf{t}_{i}}\\
		{{\mathbf{0}}}&{1}\\
		\end{bmatrix}\right)^{-1}\\
		& \qquad \qquad \quad \ \ \begin{bmatrix}{\mathbf{R}_{y}}&{\tilde{\mathbf{t}}}\\
		{{\mathbf{0}}}&{1}
		\end{bmatrix}
		\left(\begin{bmatrix}{\mathbf{R}_{\text{imu}}}&{\mathbf{0}}\\
		{{\mathbf{0}}}&{1}\\
		\end{bmatrix}\begin{bmatrix}{\mathbf{R}_{i}}&{\mathbf{t}_{i}}\\
		{{\mathbf{0}}}&{1}\\
		\end{bmatrix}\right).
		\end{aligned}
		\label{eq:transformationmatrix_Ev}
		\end{equation}}\\	
	we denote that
	{\begin{eqnarray}
		\begin{aligned}		
		&\begin{bmatrix}
		{\tilde{\mathbf{R}}_{\text{imu}}}&{\tilde{\mathbf{t}}_{\text{imu}}}\\
		{{\mathbf{0}}}&{1}\\
		\end{bmatrix} = \begin{bmatrix}{\mathbf{R}_{\text{imu}}}&{\mathbf{0}}\\
		{{\mathbf{0}}}&{1}\\
		\end{bmatrix}\begin{bmatrix}{\mathbf{R}_{i}}&{\mathbf{t}_{i}}\\
		{{\mathbf{0}}}&{1}\\
		\end{bmatrix},\\	
		&\begin{bmatrix}
		{\tilde{\mathbf{R}}'_{\text{imu}}}&{\tilde{\mathbf{t}}'_{\text{imu}}}\\
		{{\mathbf{0}}}&{1}\\
		\end{bmatrix} = \begin{bmatrix}{\mathbf{R}'_{\text{imu}}}&{\mathbf{0}}\\
		{{\mathbf{0}}}&{1}\\
		\end{bmatrix}\begin{bmatrix}{\mathbf{R}_{i}}&{\mathbf{t}_{i}}\\
		{{\mathbf{0}}}&{1}\\
		\end{bmatrix}.
		\end{aligned}
		\label{eq:R_imuNew}
		\end{eqnarray}}
	
	\noindent 
	By substituting Eq.~\eqref{eq:R_imuNew} into Eq.~\eqref{eq:transformationmatrix_Ev}, we obtain
	{\begin{equation}
		\small
		\begin{aligned}
		&\begin{bmatrix}
		{\mathbf{R}_{Ci}}&{\mathbf{t}_{Ci}}\\
		{{\mathbf{0}}}&{1}\\
		\end{bmatrix}\\
		&=\begin{bmatrix}{({\tilde{\mathbf{R}}'_{\text{imu}}})^T{\mathbf{R}_{y}}{\tilde{\mathbf{R}}_{\text{imu}}}}& {{({\tilde{\mathbf{R}}'_{\text{imu}}})^T}({\mathbf{R}_{y}}{\tilde{\mathbf{t}}_{\text{imu}}}+{\tilde{\mathbf{t}}}-{\tilde{\mathbf{t}}'_{\text{imu}}})}\\
		{{\mathbf{0}}}&{1}\\
		\end{bmatrix}.
		\end{aligned}
		\label{eq:transformationmatrix_Ev2}
		\end{equation}}
	
	\noindent 
	Essential matrix $\mathbf{E}$ between the two frames is given as
	\begin{equation}
	\begin{aligned}
	\mathbf{E} = [\mathbf{t}_{Ci}]_{\times}\mathbf{R}_{Ci} = {({\tilde{\mathbf{R}}'_{\text{imu}}})^T}[{\tilde{\mathbf{R}}'_{\text{imu}}}\mathbf{t}_{Ci}]_{\times}{{\mathbf{R}_{y}}{\tilde{\mathbf{R}}_{\text{imu}}}},
	\end{aligned}
	\label{eq:Ev}
	\end{equation}
	where $[{\tilde{\mathbf{R}}'_{\text{imu}}}\mathbf{t}_{Ci}]_{\times}={{\mathbf{R}_{y}}[\tilde{\mathbf{t}}_{\text{imu}}]_{\times}{\mathbf{R}_{y}^T}} + [\tilde{\mathbf{t}}]_{\times} - [\tilde{\mathbf{t}}'_{\text{imu}}]_{\times}$. By substituting Eq.~\eqref{eq:Ev} into Eq.~\eqref{eq:E6dof_Ac1}, we obtain
	{\begin{eqnarray}
		\begin{aligned}		
		({\tilde{\mathbf{R}}_{\text{imu}}^T}&{\mathbf{R}_{y}^T}{[{\tilde{\mathbf{R}}'_{\text{imu}}}\mathbf{t}_{Ci}]_{\times}^T}{{\tilde{\mathbf{R}}'_{\text{imu}}}}{\mathbf{x}}'_{ij})_{(1:2)} = 	\\ &-(\hat{\mathbf{A}}^{T}{({\tilde{\mathbf{R}}'_{\text{imu}}})^T}[{\tilde{\mathbf{R}}'_{\text{imu}}}\mathbf{t}_{Ci}]_{\times}{{\mathbf{R}_{y}}{\tilde{\mathbf{R}}_{\text{imu}}}}{\mathbf{x}}_{ij})_{(1:2)}.
		\end{aligned}
		\label{eq:Ev_Ac}
	\end{eqnarray}}
	
	\noindent We denote the normalized homogeneous image coordinates expressed in the aligned multi-camera reference frame as  $(\tilde{{\mathbf{p}}}_{ij}, {\tilde{\mathbf{p}}}'_{ij})$, which are given as
	\begin{equation}
	\tilde{{\mathbf{p}}}_{ij} = {\tilde{\mathbf{R}}_{\text{imu}}}{\mathbf{x}}_{ij},\qquad
	\tilde{{\mathbf{p}}}'_{ij} = {{\tilde{\mathbf{R}}'_{\text{imu}}}}{\mathbf{x}}'_{ij}.
	\label{eq:Ev_alignedimage}
	\end{equation}
	Based on the above equation, Eq.~\eqref{eq:Ev_Ac} is rewritten as
	\begin{equation}
	\begin{aligned}		
	&({\tilde{\mathbf{R}}_{\text{imu}}^T}([{\tilde{\mathbf{t}}_{\text{imu}}}]_{\times}{\mathbf{R}_{y}^T}  +   {\mathbf{R}_{y}^T}[{\tilde{\mathbf{t}}}]_{\times} - {\mathbf{R}_{y}^T}[{\tilde{\mathbf{t}}'_{\text{imu}}}]_{\times}){\tilde{{\mathbf{p}}}'_{ij}})_{(1:2)} = \\
	&(\hat{\mathbf{A}}^{T}{({\tilde{\mathbf{R}}'_{\text{imu}}})^T}({\mathbf{R}_{y}}[{\tilde{\mathbf{t}}_{\text{imu}}}]_{\times} + [{\tilde{\mathbf{t}}}]_{\times}{\mathbf{R}_{y}} - [{\tilde{\mathbf{t}}_{\text{imu}}}]_{\times}{\mathbf{R}_{y}}){{\tilde{{\mathbf{p}}}}_{ij}})_{(1:2)}
	\end{aligned}
	\label{eq:Ev_Ac2}
	\end{equation}
	
	\subsection{Solution by Reduction to a Single Polynomial}
	Based on Eqs.~\eqref{eq:GECSIMU} and \eqref{eq:Ev_Ac2}, we get an equation system of three polynomials for four unknowns $q_y$, $\tilde{t}_x$, $\tilde{t}_y$ and $\tilde{t}_z$. Recall that there are three independent constraints provided by one AC. Thus, one more equation is required which can be taken from a second AC. In principle, one arbitrary equation can be chosen from Eqs.~\eqref{eq:GECSIMU} and \eqref{eq:Ev_Ac2}, for example, three constraints of the first AC, and the first constraint of the second AC are stacked into 4 equations in 4 unknowns:
	\begin{equation} 
	\frac{1}{1+{q_y^2}}\underbrace{\begin{bmatrix}
		{\tilde{M}_{11}}&{\tilde{M}_{12}}&{\tilde{M}_{13}}&{\tilde{M}_{14}}\\
		{\tilde{M}_{21}}&{\tilde{M}_{22}}&{\tilde{M}_{23}}&{\tilde{M}_{24}}\\
		{\tilde{M}_{31}}&{\tilde{M}_{32}}&{\tilde{M}_{33}}&{\tilde{M}_{34}}\\
		{\tilde{M}_{41}}&{\tilde{M}_{42}}&{\tilde{M}_{43}}&{\tilde{M}_{44}}
		\end{bmatrix} }_{\tilde{\mathbf{M}}\left( {q_y} \right)}
	\begin{bmatrix}
	{{\tilde{t}_x}}\\
	{{\tilde{t}_y}}\\
	{{\tilde{t}_z}}\\
	1
	\end{bmatrix}  = {\mathbf{0}},
	\label{eq:euq_Ev1}
	\end{equation}
	where the elements $\tilde{M}_{ij} (i=1,\ldots,4; j=1,\ldots,4)$ of the coefficient matrix $\tilde{\mathbf{M}}({q_y})$ are formed by the polynomial coefficients and one unknown variable $q_y$, see supplementary material for details. Since $\tilde{\mathbf{M}}({q_y})/(1+{q_y^2})$ is a square matrix, Eq.~\eqref{eq:euq_Ev1} has a non-trivial solution only if the $\det (\tilde{\mathbf{M}}({q_y})/(1+{q_y^2})) = 0$. The expansion of the determinant equation gives a 6-degree univariate polynomial:
	{\begin{equation}
		\begin{aligned}	  	
		\quot(\textstyle \sum_{i=0}^8 w_i q_y^i, {q_y^2}+1) = 0,
		\end{aligned}
		\label{eq:euq_Evq}
	\end{equation}}
	
	\vspace{-12pt}
	\noindent 
	where $\tilde{w}_{0},\ldots,\tilde{w}_{8}$ are formed by two Pl\"{u}cker line correspondences and two affine transformations between the corresponding feature points.
	
	This univariate polynomial leads to a maximum of 6 solutions. Equation~\eqref{eq:euq_Evq} can be efficiently solved by the companion matrix method~\cite{cox2013ideals} or Sturm bracketing method~\cite{nister2004efficient}. Once $q_y$ has been obtained, the rotation matrix $\mathbf{R}_{y}$ is recovered from Eq.~\eqref{eq:Ryt1}.
	For the relative pose between two multi-camera reference frames at time $k$ and $k+1$, the rotation matrix $\mathbf{R}$ is recovered from Eq.~\eqref{eq:Rv} and the translation is computed by $\mathbf{t} =  ({{{\mathbf{R}}}'_{\text{imu}}})^T \mathbf{\tilde{t}}$.
	
	\section{\label{sec:experiments}Experiments}
	In this section, we conduct extensive experiments on both synthetic and real-world data to evaluate the performance of the proposed methods. Our solvers are compared with state-of-the-art techniques.
	
	For relative pose estimation under planar motion, the solvers using one AC and two ACs proposed in Section~\ref{sec:planarmotion} are referred to as \texttt{1AC~plane} method and \texttt{2AC~plane} method, respectively. The accuracy of \texttt{1AC~plane} and \texttt{2AC~plane} are compared with the methods \texttt{17pt-Li}~\cite{li2008linear}, \texttt{8pt-Kneip}~\cite{kneip2014efficient}, \texttt{6pt-Stew{\'e}nius}~\cite{henrikstewenius2005solutions} and \texttt{6AC-Ventura}~\cite{alyousefi2020multi}.
	
	For relative pose estimation with known vertical direction, the solver proposed in Section~\ref{sec:knownverticaldirection} is referred to as \texttt{2AC vertical} method. We compare the accuracy of \texttt{2AC vertical} with the methods \texttt{17pt-Li}~\cite{li2008linear}, \texttt{8pt-Kneip}~\cite{kneip2014efficient}, \texttt{6pt-Stew{\'e}nius}~\cite{henrikstewenius2005solutions},  \texttt{4pt-Lee}~\cite{hee2014relative}, \texttt{4pt-Sweeney}~\cite{sweeney2014solving}, \texttt{4pt-Liu}~\cite{liu2017robust} and \texttt{6AC-Ventura}~\cite{alyousefi2020multi}.
	
	A single run of the proposed solvers \texttt{1AC~plane}, \texttt{2AC~plane} and \texttt{2AC vertical} take 3.6, 3.6 and 17.8~$\mu s$ in C++, respectively. Due to space limitations, the efficiency comparison and stability study are provided in the supplementary material. In the experiments, all the solvers are integrated within RANSAC to reject outliers. For the point-based solvers, only the point coordinates of ACs are used. The relative pose which produces the highest number of inliers is chosen. The confidence of RANSAC is set to 0.99 and an inlier threshold angle is set to $0.1^\circ$ by following the definition in OpenGV~\cite{kneip2014opengv}. We also show the feasibility of our methods on the \texttt{KITTI} dataset~\cite{geiger2013vision}. This experiment demonstrates that our methods are well suited for visual odometry in road driving scenarios.     
	
	\subsection{Experiments on Synthetic Data}
	We made a simulated 2-camera rig system by following the KITTI autonomous driving platform. The baseline length between two simulated cameras is set to 1 meter and the cameras are installed at different altitude. The multi-camera reference frame is set at the center of the camera rig and the translation between two multi-camera reference frames is 3 meters. The resolution of the cameras is 640 $\times$ 480 pixels and the focal lengths are 400 pixels. The principal points are set to the image center (320, 240).
	
	The synthetic scene is composed of a ground plane and 50 random planes. All 3D planes are randomly generated within the range of -5 to 5 meters (along axes X and Y), and 10 to 20 meters (Z-axis direction), that are expressed in the respective axis of the multi-camera reference frame. We choose 50 ACs from the ground plane and an AC from each random plane randomly, thus, having a total of 100 ACs. For each AC, a random 3D point from a plane is reprojected onto two cameras to get the image point pair. The corresponding affine transformation is obtained by the following procedure. First, four points are chosen as the vertices of a square in view~1, where the center of the square is the point coordinates of AC. The side length of the square is set as $20$ or $40$~pixels. A larger side length causes smaller noise of affine transformation. Second, the four corresponding points in view~2 are calculated by the ground truth homography. Third, four sampled point pairs are contaminated by Gaussian noise, which is similar to the noise added to the coordinates of image point pair. Fourth, the noisy homography matrix is estimated using the four sampled point pairs. The noisy affine transformation is the first-order approximation of the noisy homography matrix. This procedure enables an indirect but geometrically interpretable way of adding noise to the affine transformation~\cite{barath2019homography}. 
	
	A total of 1000 trials are carried out in the synthetic experiment. In each test, 100 ACs are generated randomly. The ACs for the methods are selected randomly and the error is measured on the relative pose which produces the most inliers within the RANSAC scheme. This also allows us to select the best candidate from multiple solutions by counting their inliers in a RANSAC-like procedure. The median of errors are used to assess the rotation and translation error. The rotation error is computed as the angular difference between the ground truth rotation and the estimated rotation: ${\varepsilon _{\bf{R}}} = \arccos ((\trace({\mathbf{R}_{gt}}{{\mathbf{R}^T}}) - 1)/2)$, where $\mathbf{R}_{gt}$ and ${\mathbf{R}}$ are the ground truth and estimated rotation matrices. Following the definition in~\cite{quan1999linear,hee2014relative}, the translation error is defined as: ${\varepsilon _{\bf{t}}} = 2\left\| ({{\mathbf{t}_{gt}}}-{\mathbf{t}})\right\|/(\left\| {\mathbf{t}_{gt}} \right\| + \left\| {{\mathbf{t}}} \right\|)$, where $\mathbf{t}_{gt}$ and ${\mathbf{t}}$ are the ground truth and estimated translations.    
	\subsubsection{Planar Motion Estimation}
	\vspace{1pt} 
	\begin{figure}[tbp]
		\begin{center}
			\includegraphics[width=0.9\linewidth]{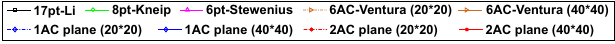}\\
			\subfigure[\scriptsize{${\varepsilon_{\bf{R}}}$ with image noise}]
			{
				\includegraphics[width=0.303\linewidth]{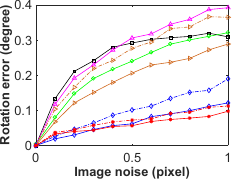}
			}
			\vspace{-1pt} 
			\subfigure[\scriptsize{${\varepsilon_{\bf{t}}}$ with image noise}]
			{
				\includegraphics[width=0.303\linewidth]{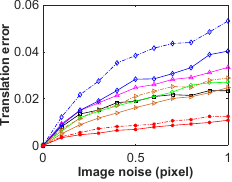}
			}
			\subfigure[\scriptsize{Translation direction error with image noise}]
			{
				\includegraphics[width=0.303\linewidth]{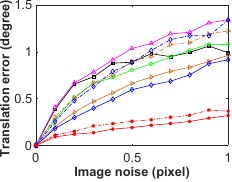}
			}
			\subfigure[\scriptsize{${\varepsilon_{\bf{R}}}$ with planar motion noise}]
			{
				\includegraphics[width=0.303\linewidth]{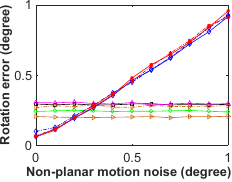}
			}
			\vspace{-1pt}
			\subfigure[\scriptsize{${\varepsilon_{\bf{t}}}$ with planar motion noise}]
			{
				\includegraphics[width=0.303\linewidth]{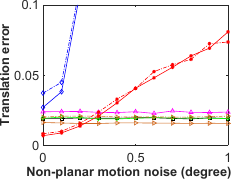}
			}
			\subfigure[\scriptsize{Translation direction error with planar motion noise}]
			{
				\includegraphics[width=0.303\linewidth]{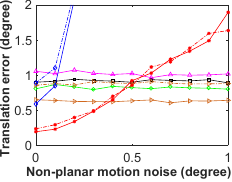}
			}
		\end{center}
		\caption{Rotation and translation error under planar motion. (a--c): varying image noise under perfect planar motion. (d--f): varying planar motion noise and fixed $1.0$ pixel std.\ image noise.}
		\vspace{-2pt} 
		\label{fig:RT_planar}
	\end{figure}
	In this scenario, the planar motion of the multi-camera system is described by ($\theta$, $\phi$), see Fig.~\ref{fig:Specialcases}(a). The magnitudes of both angles ranges from $-10^\circ$ to $10^\circ$. 
	We use Gaussian image noise with a standard deviation ranging from $0$ to $1$ pixel. Fig.~\ref{fig:RT_planar}(a--c) shows the performance of the proposed \texttt{1AC~plane} and \texttt{2AC~plane} methods against image noise. Since the noise magnitude of affine transformation is influenced by the support region of sampled points, the AC-based methods have better performance with larger support region at the same magnitude of image noise. It can be seen that \texttt{2AC~plane} performs better than the other compared methods under perfect planar motion, even though the size of the square is $20$~pixels. The \texttt{1AC~plane} method performs better than the PC-based methods and the \texttt{6AC-Ventura} method in rotation estimation, but has worse performance in translation estimation. As shown in Fig.~\ref{fig:RT_planar}(c), we plot the translation direction error as an additional evaluation. It is interesting to see that when the side length of the square is $40$~pixels, the \texttt{1AC~plane} method performs better than the PC-based methods and the \texttt{6AC-Ventura} method in translation direction estimation. 
	\begin{table*}[tbp]
		\begin{center}
			\setlength{\tabcolsep}{0.9mm}{
				\scalebox{0.85}{
					\begin{tabular}{|c||c|c|c|c|c|c|c|c|c|}
						\hline
						\multirow{2}{*}{\small{Sequence}} &  \footnotesize{17pt-Li~\cite{li2008linear}} & \footnotesize{8pt-Kneip~\cite{kneip2014efficient}}  &  \footnotesize{6pt-}{\footnotesize{St.}}~\footnotesize{\cite{henrikstewenius2005solutions}}   &  \footnotesize{4pt-Lee~\cite{hee2014relative}} & \footnotesize{4pt-Sw.~\cite{sweeney2014solving}}& \footnotesize{4pt-Liu~\cite{liu2017robust}}& \footnotesize{6AC-Ven.~\cite{alyousefi2020multi}}& \footnotesize{\textbf{2AC~plane}}& \footnotesize{\textbf{2AC~vertical}} \\
						\cline{2-10}
						& ${\varepsilon _{\bf{R}}}$\qquad\ ${\varepsilon _{\bf{t}}}$      &  ${\varepsilon _{\bf{R}}}$\qquad\ ${\varepsilon _{\bf{t}}}$      &   ${\varepsilon _{\bf{R}}}$\qquad\ ${\varepsilon _{\bf{t}}}$     &   ${\varepsilon _{\bf{R}}}$\qquad\ ${\varepsilon _{\bf{t}}}$  &   ${\varepsilon _{\bf{R}}}$\qquad\ ${\varepsilon _{\bf{t}}}$   &   ${\varepsilon _{\bf{R}}}$\qquad\ ${\varepsilon _{\bf{t}}}$ &   ${\varepsilon _{\bf{R}}}$\qquad\ ${\varepsilon _{\bf{t}}}$  &   ${\varepsilon _{\bf{R}}}$\qquad\ ${\varepsilon _{\bf{t}}}$  &   ${\varepsilon _{\bf{R}}}$\qquad\ ${\varepsilon _{\bf{t}}}$\\
						\hline
						\small{00 (4541 images)}&                   0.139 \ 2.412 &  0.130  \  2.400& 0.229 \ 4.007 & 0.065 \ 2.469 & 0.050 \   2.190   & 0.066 \ 2.519 & 0.142  \  2.499  & 0.280  \  2.243 &\textbf{0.031} \ \textbf{1.738}     \\
						\rowcolor{gray!10}\small{01 (1101 images)}& 0.158 \ 5.231 &  0.171  \  4.102& 0.762 \ 41.19 & 0.137 \ 4.782 & 0.125 \   11.91   & 0.105 \ 3.781 & 0.146  \  3.654  & 0.168  \  2.486 &\textbf{0.025} \ \textbf{1.428}     \\
						\small{02 (4661 images)}&                   0.123 \ 1.740 &  0.126  \  1.739& 0.186 \ 2.508 & 0.057 \ 1.825 & 0.044 \   1.579   & 0.057 \ 1.821 & 0.121  \  1.702  & 0.213  \  1.975 &\textbf{0.030} \ \textbf{1.558}     \\
						\rowcolor{gray!10}\small{03 \ \ (801 images)}& 0.115 \ 2.744&0.108  \  2.805& 0.265 \ 6.191 & 0.064 \ 3.116 & 0.069 \   3.712   & 0.062 \ 3.258 & 0.113  \  2.731  & 0.238  \  \textbf{1.849} &\textbf{0.037} \ 1.888     \\
						\small{04 \ \ (271 images)}&                   0.099 \ 1.560&0.116  \  1.746& 0.202 \ 3.619 & 0.050 \ 1.564 & 0.051 \   1.708   & 0.045 \ 1.635 & 0.100  \  1.725  & 0.116  \  1.768 &\textbf{0.020} \ \textbf{1.228}     \\
						\rowcolor{gray!10}\small{05 (2761 images)}& 0.119 \ 2.289 &  0.112  \  2.281& 0.199 \ 4.155 & 0.054 \ 2.337 & 0.052 \   2.544   & 0.056 \ 2.406 & 0.116  \  2.273  & 0.185  \  2.354 &\textbf{0.022} \ \textbf{1.532}     \\
						\small{06 (1101 images)}&                   0.116 \ 2.071 &  0.118  \  1.862& 0.168 \ 2.739 & 0.053 \ 1.757 & 0.092 \   2.721   & 0.056 \ 1.760 & 0.115  \  1.956  & 0.137  \  2.247 &\textbf{0.023} \ \textbf{1.303}     \\
						\rowcolor{gray!10}\small{07 (1101 images)}& 0.119 \ 3.002 &  0.112  \  3.029& 0.245 \ 6.397 & 0.058 \ 2.810 & 0.065 \   4.554   & 0.054 \ 3.048 & 0.137  \  2.892  & 0.173  \  2.902 &\textbf{0.023} \ \textbf{1.820}     \\
						\small{08 (4071 images)}&                   0.116 \ 2.386 &  0.111  \  2.349& 0.196 \ 3.909 & 0.051 \ 2.433 & 0.046 \   2.422   & 0.053 \ 2.457 & 0.108  \  2.344  & 0.203  \  2.569 &\textbf{0.024} \ \textbf{1.911}     \\
						\rowcolor{gray!10}\small{09 (1591 images)}& 0.133 \ 1.977 &  0.125  \  1.806& 0.179 \ 2.592 & 0.056 \ 1.838 & 0.046 \   1.656   & 0.058 \ 1.793 & 0.124  \  1.876  & 0.189  \  1.997 &\textbf{0.027} \ \textbf{1.440}     \\
						\small{10 (1201 images)}&                   0.127 \ 1.889 &  0.115  \  1.893& 0.201 \ 2.781 & 0.052 \ 1.932 & 0.040 \   1.658   & 0.058 \ 1.888 & 0.203  \  2.057  & 0.223  \  2.296 &\textbf{0.025} \ \textbf{1.586}     \\
						\hline					
			\end{tabular}}}
		\end{center}
		\caption{Rotation and translation error on \texttt{KITTI} sequences (unit: degree).}
		\label{VerticalRTErrror}
	\end{table*}

	We also evaluate the accuracy of the proposed methods \texttt{1AC~plane} and \texttt{2AC~plane} for increasing planar motion noise. 
	To test such noise, we added a small randomly generated X-axis, Z-axis rotation and a YZ-plane translation~\cite{choi2018fast} to the motion of the multi-camera system. The magnitude of non-planar motion noise ranges from $0^\circ$ to $1^\circ$ and the standard deviation of the image noise is set to $0.5$ pixel. Figures~\ref{fig:RT_planar}(d--f) show the performance of the proposed \texttt{1AC~plane} method and \texttt{2AC~plane} method against planar motion noise. Methods \texttt{17pt-Li}, \texttt{8pt-Kneip}, \texttt{6pt-Stew{\'e}nius} and \texttt{6AC-Ventura} deal with the 6DOF motion case and, thus they are not affected by the noise in the planarity assumption. It can be seen that the rotation accuracy of the \texttt{2AC~plane} method performs better than comparative methods when the planar motion noise is less than $0.2^\circ$. Since the estimation accuracy of translation direction of the \texttt{2AC~plane} method in Fig.~\ref{fig:RT_planar}(f) performs satisfactory, the main reason for poor performance of translation estimation is that the metric scale estimation is sensitive to the planar motion noise. In comparison with the \texttt{2AC~plane} method, the \texttt{1AC~plane} method has similar performance in rotation estimation, but performs poorly in translation estimation. The translation accuracy decreases significantly with the increase of the planar motion noise.
	
	Both the \texttt{1AC~plane} and \texttt{2AC~plane} methods have a significant computational advantage over comparative methods, because the efficient solver for 4-degree polynomial equation takes only about 3.6~$\mu s$. Moreover, since only a single AC is required, the \texttt{1AC~plane} method has the advantage of detecting a correct inlier set efficiently, which can then be used for accurate motion estimation with non-linear optimization. See supplementary material for details. 
	
	\vspace{-5pt} 
	\subsubsection{Motion with Known Vertical Direction}
	In this set of experiments, the translation direction of two multi-camera reference frames is chosen to produce either forward, sideways or random motions. The second reference frame is rotated around three axes randomly with angles ranging from $-10^\circ$ to $10^\circ$. Assuming known roll and pitch angles, the multi-camera reference frame is aligned with the vertical direction. Due to space limitations, we only show the results for random motion. Other results are in the supplementary material. Figs.~\ref{fig:RT_1AC}(a) and (d) show the performance of \texttt{2AC~vertical} against image noise with perfect IMU data.
	The proposed method is robust to image noise and performs better than the other methods. 
	
	\begin{figure}[tbp]
		\begin{center}
			\vspace{12pt}
			\includegraphics[width=0.95\linewidth]{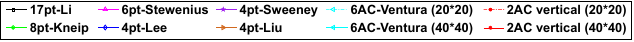}\\
			\subfigure[\scriptsize{${\varepsilon _{\bf{R}}}$ with image noise}]
			{
				\includegraphics[width=0.303\linewidth]{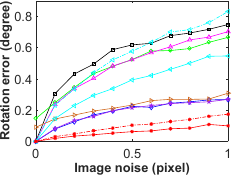}
			}
			\subfigure[\scriptsize{${\varepsilon _{\bf{R}}}$ with pitch angle noise}]
			{
				\includegraphics[width=0.303\linewidth]{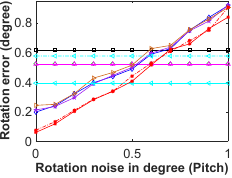}
			}
			\subfigure[\scriptsize{${\varepsilon _{\bf{R}}}$ with roll angle noise}]
			{
				\includegraphics[width=0.303\linewidth]{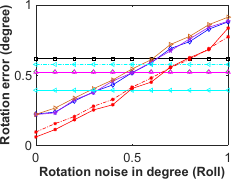}
			}
			\subfigure[\scriptsize{${\varepsilon _{\bf{t}}}$ with image noise}]
			{
				\includegraphics[width=0.303\linewidth]{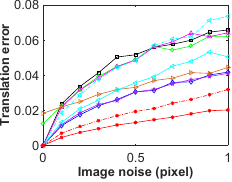}
			}
			\subfigure[\scriptsize{${\varepsilon _{\bf{t}}}$ with pitch angle noise}]
			{
				\includegraphics[width=0.303\linewidth]{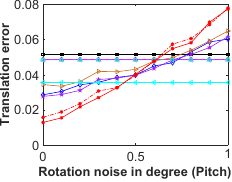}
			}
			\subfigure[\scriptsize{${\varepsilon _{\bf{t}}}$ with roll angle noise}]
			{
				\includegraphics[width=0.303\linewidth]{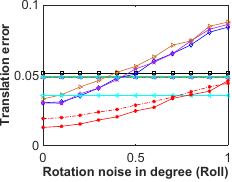}
			}
		\end{center}
		\caption{Rotation and translation error under random motion with known vertical direction. Upper row: rotation error. Bottom row: translation error. (a,d): varying image noise. (b,e) and (c,f): varying IMU angle noise and fixed $1.0$ pixel std.\ image noise.}
		\label{fig:RT_1AC}
	\end{figure}

	\begin{table*}[t]
	\begin{center}
		\setlength{\tabcolsep}{1.0mm}{
			\scalebox{0.9776}{
				\begin{tabular}{|c||c|c|c|c|c|c|c|c|c|}
					\hline
					\small{Methods} &  \footnotesize{17pt-Li~\cite{li2008linear}} & \footnotesize{8pt-Kneip~\cite{kneip2014efficient}} &  \footnotesize{6pt-}{\footnotesize{St.}}~\footnotesize{\cite{henrikstewenius2005solutions}}    &\footnotesize{4pt-Lee~\cite{hee2014relative}} & \footnotesize{4pt-Sw.~\cite{sweeney2014solving}}& \footnotesize{4pt-Liu~\cite{liu2017robust}} & \footnotesize{6AC-Ven.~\cite{alyousefi2020multi}} & \footnotesize{\textbf{2AC~plane}}& \footnotesize{\textbf{2AC~vertical}}  \\
					\hline
					\small{Mean time }& 52.82 & 10.36 & 79.76& 0.85& 0.63& 0.45& 6.83 & \textbf{0.07} & 0.09\\
					\hline
					\small{Standard deviation}& 2.62 & 1.59 & 4.52& 0.093& 0.057& 0.058& 0.61 &\textbf{0.0071} & 0.0086\\
					\hline
		\end{tabular}}}
	\end{center}
	\caption{Runtime of RANSAC averaged over \texttt{KITTI} sequences combined with different solvers (unit:~$s$).}
	\label{RANSACTime}
	\end{table*}
	
	Figs.~\ref{fig:RT_1AC}(b,e) and (c,f) show the performance of \texttt{2AC~vertical} against IMU noise in the random motion case, while the standard deviation of the image noise is fixed at $0.5$ pixel. Note that the methods \texttt{17pt-Li}, \texttt{8pt-Kneip}, \texttt{6pt-Stew{\'e}nius} and \texttt{6AC-Ventura} are not influenced by IMU noise, because these methods do not use the known vertical direction as a prior. 
	The methods \texttt{4pt-Lee}, \texttt{4pt-Sweeney} and \texttt{4pt-Liu} use the known vertical direction as a prior. It is interesting to see that the proposed method outperforms the comparative methods in the random motion case, even though the IMU noise is around $0.4^\circ$. The results under forward and sideways motion also demonstrate that the \texttt{2AC~vertical} method performs basically better than all comparative methods against image noise and provides comparable accuracy for increasing IMU noise.

	\subsection{Experiments on Real Data}
	We test the performance of our methods on the \texttt{KITTI} dataset~\cite{geiger2013vision} that consists of successive video frames from a forward facing stereo camera. The ground truth pose is provided from the built-in GPS/IMU units. We ignore the overlap in their fields of view and treat it as a general multi-camera system. The sequences labeled from 0 to 10, which have ground truth, are used for the evaluation. Therefore, the methods were tested on a total of 23000 image pairs. The ACs between consecutive frames in each camera are established by applying the ASIFT~\cite{morel2009asift} detector. 
	The extraction of ACs can also be sped up by MSER~\cite{matas2004robust}, GPU acceleration, or approximating ACs from SIFT features for subsequent video frames.
	The ACs across the two cameras are not matched and the metric scale is not estimated as the movement between consecutive frames is small. Besides, integrating the acceleration over time from an IMU is more suitable for recovering the scale~\cite{NutziWeiss-411}. All the solvers have been integrated into a RANSAC scheme.  
	
	The proposed methods \texttt{2AC~plane} and \texttt{2AC vertical} are compared against \texttt{17pt-Li}~\cite{li2008linear}, \texttt{8pt-Kneip}~\cite{kneip2014efficient}, \texttt{6pt-Stew{\'e}nius}~\cite{henrikstewenius2005solutions}, \texttt{4pt-Lee}~\cite{hee2014relative}, \texttt{4pt-Sweeney}~\cite{sweeney2014solving}, \texttt{4pt-Liu}~\cite{liu2017robust} and \texttt{6AC-Ventura}~\cite{alyousefi2020multi}. Since the \texttt{KITTI} dataset is captured by a stereo rig with both cameras having the same altitude, that is a degenerate case for the \texttt{1AC~plane} method, it is not performed in the experiment. For the \texttt{2AC~plane} method, the results are also compared to the ground truth of the 6DOF relative pose, even though this method only estimates two angles ($\theta$, $\phi$) with the plane motion assumption. For the \texttt{2AC vertical} method, the roll and pitch angles obtained from the GPS/IMU units are used to align the multi-camera reference frame with the vertical direction~\cite{SaurerVasseur-457,Guan2020CVPR,Li2020Relative}. To ensure the fairness of the experiment, the roll and pitch angles are also provided for the methods \texttt{4pt-Lee}~\cite{hee2014relative}, \texttt{4pt-Sweeney}~\cite{sweeney2014solving} and \texttt{4pt-Liu}~\cite{liu2017robust}, which use the known vertical direction as a prior. Table~\ref{VerticalRTErrror} shows the results of the rotation and translation estimation. The median error for each individual sequence is used to evaluate the estimation accuracy. The runtime of RANSAC averaged over \texttt{KITTI} sequences combined with different solvers is shown in Table~\ref{RANSACTime}. The reported runtimes include the robust relative pose estimation without feature extraction, \emph{i.e.}, recovering the relative pose by RANSAC combined with a minimal solver.
	
	The proposed \texttt{2AC vertical} method offers the best overall performance among all the methods. The \texttt{6pt-Stew{\'e}nius} method performs poorly on sequence 01, because this sequence is a highway with few tractable close objects, and this method always fails to select the best candidate from multiple solutions under forward motion in the RANSAC scheme. Besides, it is interesting to see that the translation accuracy of the \texttt{2AC~plane} method basically outperforms the \texttt{6pt-Stew{\'e}nius} method, even though the planar motion assumption does not fit the \texttt{KITTI} dataset well. To visualize the comparison results, the estimated trajectory for sequence 00 is plotted in the supplementary material. Due to the benefits of computational efficiency, both the \texttt{2AC~plane} method and the \texttt{2AC vertical} method are quite suitable for finding a correct inlier set, which is then used for accurate motion estimation in visual odometry. 
	
	\section{\label{sec:conclusion}Conclusion}
	By exploiting the geometric constraints which interprets the relationship of ACs and the generalized camera model, we have proposed three solutions for the relative pose estimation of a multi-camera system. Under the planar motion assumption, we present two solvers to recover the relative pose of a multi-camera system, including a minimal solver using a single AC and a solver based on two ACs. In addition, a minimal solution with two ACs is proposed to solve for the relative pose of the multi-camera system with known vertical direction. 
	Both planar motion and known vertical direction assumptions are realistic in autonomous driving scenes. We evaluate the proposed solvers on synthetic data and real image sequence datasets. The experimental results clearly showed that the proposed methods provide better efficiency and accuracy for relative pose estimation in comparison to state-of-the-art methods.
	
	\section*{Acknowledgments}
	This work has been partially funded by the National Natural Science Foundation of China (11902349, 11727804).   
	
	{\small
		\bibliographystyle{ieee_fullname}
		\bibliography{myBibGuan}

\begin{thebibliography}{10}\itemsep=-1pt

\bibitem{Agarwal2017}
Sameer Agarwal, Hon-Leung Lee, Bernd Sturmfels, and Rekha~R. Thomas.
\newblock On the existence of epipolar matrices.
\newblock {\em International Journal of Computer Vision}, 121(3):403--415,
  2017.

\bibitem{alyousefi2020multi}
Khaled Alyousefi and Jonathan Ventura.
\newblock Multi-camera motion estimation with affine correspondences.
\newblock In {\em International Conference on Image Analysis and Recognition},
  pages 417--431, 2020.

\bibitem{barath2018five}
Daniel Barath.
\newblock Five-point fundamental matrix estimation for uncalibrated cameras.
\newblock In {\em IEEE Conference on Computer Vision and Pattern Recognition},
  pages 235--243, 2018.

\bibitem{barath2018efficient}
Daniel Barath and Levente Hajder.
\newblock Efficient recovery of essential matrix from two affine
  correspondences.
\newblock {\em IEEE Transactions on Image Processing}, 27(11):5328--5337, 2018.

\bibitem{barath2019homography}
Daniel Barath and Zuzana Kukelova.
\newblock Homography from two orientation-and scale-covariant features.
\newblock In {\em IEEE International Conference on Computer Vision}, pages
  1091--1099, 2019.

\bibitem{barath2020making}
Daniel Barath, Michal Polic, Wolfgang F{\"o}rstner, Torsten Sattler, Tomas
  Pajdla, and Zuzana Kukelova.
\newblock Making affine correspondences work in camera geometry computation.
\newblock In {\em European Conference on Computer Vision}, 2020.

\bibitem{Bay2008346}
Herbert Bay, Andreas Ess, Tinne Tuytelaars, and Luc Van~Gool.
\newblock Speeded-up robust features ({SURF}).
\newblock {\em Computer Vision and Image Understanding}, 110(3):346--359, 2008.

\bibitem{bentolila2014conic}
Jacob Bentolila and Joseph~M Francos.
\newblock Conic epipolar constraints from affine correspondences.
\newblock {\em Computer Vision and Image Understanding}, 122:105--114, 2014.

\bibitem{Caesar_2020_CVPR}
Holger Caesar, Varun Bankiti, Alex~H. Lang, Sourabh Vora, Venice~Erin Liong,
  Qiang Xu, Anush Krishnan, Yu Pan, Giancarlo Baldan, and Oscar Beijbom.
\newblock nuscenes: {A} multimodal dataset for autonomous driving.
\newblock In {\em IEEE Conference on Computer Vision and Pattern Recognition},
  pages 11621--11631, 2020.

\bibitem{choi2018fast}
Sunglok Choi and Jong-Hwan Kim.
\newblock Fast and reliable minimal relative pose estimation under planar
  motion.
\newblock {\em Image and Vision Computing}, 69:103--112, 2018.

\bibitem{clipp2008robust}
Brian Clipp, Jae-Hak Kim, Jan-Michael Frahm, Marc Pollefeys, and Richard
  Hartley.
\newblock Robust 6dof motion estimation for non-overlapping, multi-camera
  systems.
\newblock In {\em IEEE Workshop on Applications of Computer Vision}, pages
  1--8. IEEE, 2008.

\bibitem{cox2013ideals}
David Cox, John Little, and Donal O'Shea.
\newblock {\em Ideals, varieties, and algorithms: An introduction to
  computational algebraic geometry and commutative algebra}.
\newblock Springer Science \& Business Media, 2013.

\bibitem{Eichhardt2020Relative}
Iv{\'a}n Eichhardt and Daniel Barath.
\newblock Relative pose from deep learned depth and a single affine
  correspondence.
\newblock In {\em European Conference on Computer Vision}, 2020.

\bibitem{fischler1981random}
Martin~A Fischler and Robert~C Bolles.
\newblock Random sample consensus: A paradigm for model fitting with
  applications to image analysis and automated cartography.
\newblock {\em Communications of the ACM}, 24(6):381--395, 1981.

\bibitem{Fragoso_2020_CVPR}
Victor Fragoso, Joseph DeGol, and Gang Hua.
\newblock gdls*: {Generalized} pose-and-scale estimation given scale and
  gravity priors.
\newblock In {\em IEEE Conference on Computer Vision and Pattern Recognition},
  pages 2210--2219, 2020.

\bibitem{geiger2013vision}
Andreas Geiger, Philip Lenz, Christoph Stiller, and Raquel Urtasun.
\newblock Vision meets robotics: {The} {KITTI} dataset.
\newblock {\em The International Journal of Robotics Research},
  32(11):1231--1237, 2013.

\bibitem{guan2018visual}
Banglei Guan, Pascal Vasseur, C{\'e}dric Demonceaux, and Friedrich Fraundorfer.
\newblock Visual odometry using a homography formulation with decoupled
  rotation and translation estimation using minimal solutions.
\newblock In {\em IEEE International Conference on Robotics and Automation},
  pages 2320--2327, 2018.

\bibitem{Guan_ICRA2021}
Banglei Guan, Ji Zhao, Daniel Barath, and Friedrich Fraundorfer.
\newblock Efficient recovery of multi-camera motion from two affine
  correspondences.
\newblock In {\em IEEE International Conference on Robotics and Automation},
  pages 1--6, 2021.

\bibitem{Guan2020CVPR}
Banglei Guan, Ji Zhao, Zhang Li, Fang Sun, and Friedrich Fraundorfer.
\newblock Minimal solutions for relative pose with a single affine
  correspondence.
\newblock In {\em IEEE Conference on Computer Vision and Pattern Recognition},
  pages 1929--1938, 2020.

\bibitem{Guan_TCYB2021}
Banglei Guan, Ji Zhao, Zhang Li, Fang Sun, and Friedrich Fraundorfer.
\newblock Relative pose estimation with a single affine correspondence.
\newblock {\em IEEE Transactions on Cybernetics}, pages 1--12, 2021.

\bibitem{hajder2019relative}
Levente Hajder and Daniel Barath.
\newblock Relative planar motion for vehicle-mounted cameras from a single
  affine correspondence.
\newblock In {\em IEEE International Conference on Robotics and Automation},
  pages 8651--8657, 2020.

\bibitem{hane20173d}
Christian H{\"a}ne, Lionel Heng, Gim~Hee Lee, Friedrich Fraundorfer, Paul
  Furgale, Torsten Sattler, and Marc Pollefeys.
\newblock 3{D} visual perception for self-driving cars using a multi-camera
  system: Calibration, mapping, localization, and obstacle detection.
\newblock {\em Image and Vision Computing}, 68:14--27, 2017.

\bibitem{HartleyZisserman-472}
Richard Hartley and Andrew Zisserman.
\newblock {\em Multiple view geometry in computer vision}.
\newblock Cambridge University Press, 2003.

\bibitem{heng2019project}
Lionel Heng, Benjamin Choi, Zhaopeng Cui, Marcel Geppert, Sixing Hu, Benson
  Kuan, Peidong Liu, Rang Nguyen, Ye~Chuan Yeo, Andreas Geiger, Gim~Hee Lee,
  Marc Pollefeys, and Torsten Sattler.
\newblock Project {AutoVision}: Localization and 3{D} scene perception for an
  autonomous vehicle with a multi-camera system.
\newblock In {\em IEEE International Conference on Robotics and Automation},
  pages 4695--4702, 2019.

\bibitem{henrikstewenius2005solutions}
Stew{\'e}nius Henrik, Oskarsson Magnus, Kalle Astr{\"o}m, and David Nist{\'e}r.
\newblock Solutions to minimal generalized relative pose problems.
\newblock In {\em Workshop on Omnidirectional Vision in conjunction with ICCV},
  pages 1--8, 2005.

\bibitem{kim2009motion}
Jae-Hak Kim, Hongdong Li, and Richard Hartley.
\newblock Motion estimation for nonoverlapping multicamera rigs: Linear
  algebraic and {$L_{\infty}$} geometric solutions.
\newblock {\em IEEE Transactions on Pattern Analysis and Machine Intelligence},
  32(6):1044--1059, 2009.

\bibitem{kneip2014opengv}
Laurent Kneip and Paul Furgale.
\newblock Open{GV}: {A} unified and generalized approach to real-time
  calibrated geometric vision.
\newblock In {\em IEEE International Conference on Robotics and Automation},
  pages 12034--12043, 2014.

\bibitem{kneip2014efficient}
Laurent Kneip and Hongdong Li.
\newblock Efficient computation of relative pose for multi-camera systems.
\newblock In {\em IEEE Conference on Computer Vision and Pattern Recognition},
  pages 446--453, 2014.

\bibitem{kneip2016generalized}
Laurent Kneip, Chris Sweeney, and Richard Hartley.
\newblock The generalized relative pose and scale problem: View-graph fusion
  via 2{D}-2{D} registration.
\newblock In {\em IEEE Winter Conference on Applications of Computer Vision},
  pages 1--9, 2016.

\bibitem{hee2013motion}
Gim~Hee Lee, Friedrich Faundorfer, and Marc Pollefeys.
\newblock Motion estimation for self-driving cars with a generalized camera.
\newblock In {\em IEEE Conference on Computer Vision and Pattern Recognition},
  pages 2746--2753, 2013.

\bibitem{hee2014relative}
Gim~Hee Lee, Marc Pollefeys, and Friedrich Fraundorfer.
\newblock Relative pose estimation for a multi-camera system with known
  vertical direction.
\newblock In {\em IEEE Conference on Computer Vision and Pattern Recognition},
  pages 540--547, 2014.

\bibitem{Li2020Relative}
Bo Li, Evgeniy Martyushev, and Gim~Hee Lee.
\newblock Relative pose estimation of calibrated cameras with known
  $\mathrm{SE}(3)$ invariants.
\newblock In {\em European Conference on Computer Vision}, 2020.

\bibitem{li2008linear}
Hongdong Li, Richard Hartley, and Jae-hak Kim.
\newblock A linear approach to motion estimation using generalized camera
  models.
\newblock In {\em IEEE Conference on Computer Vision and Pattern Recognition},
  pages 1--8, 2008.

\bibitem{lim2010estimating}
John Lim, Nick Barnes, and Hongdong Li.
\newblock Estimating relative camera motion from the antipodal-epipolar
  constraint.
\newblock {\em IEEE Transactions on Pattern Analysis and Machine Intelligence},
  32(10):1907--1914, 2010.

\bibitem{liu2017robust}
Liu Liu, Hongdong Li, Yuchao Dai, and Quan Pan.
\newblock Robust and efficient relative pose with a multi-camera system for
  autonomous driving in highly dynamic environments.
\newblock {\em IEEE Transactions on Intelligent Transportation Systems},
  19(8):2432--2444, 2017.

\bibitem{Lowe2004Distinctive}
David~G. Lowe.
\newblock Distinctive image features from scale-invariant keypoints.
\newblock {\em International Journal of Computer Vision}, 60(2):91--110, 2004.

\bibitem{Martyushev2020Efficient}
Evgeniy Martyushev and Bo Li.
\newblock Efficient relative pose estimation for cameras and generalized
  cameras in case of known relative rotation angle.
\newblock {\em Journal of Mathematical Imaging and Vision}, (10), 2020.

\bibitem{matas2004robust}
Jiri Matas, Ondrej Chum, Martin Urban, and Tom{\'a}s Pajdla.
\newblock Robust wide-baseline stereo from maximally stable extremal regions.
\newblock {\em Image and Vision Computing}, 22(10):761--767, 2004.

\bibitem{morel2009asift}
Jean-Michel Morel and Guoshen Yu.
\newblock {ASIFT}: A new framework for fully affine invariant image comparison.
\newblock {\em SIAM Journal on Imaging Sciences}, 2(2):438--469, 2009.

\bibitem{nister2004efficient}
David Nist{\'e}r.
\newblock An efficient solution to the five-point relative pose problem.
\newblock {\em IEEE Transactions on Pattern Analysis and Machine Intelligence},
  26(6):0756--777, 2004.

\bibitem{NutziWeiss-411}
Gabriel Nützi, Stephan Weiss, Davide Scaramuzza, and Roland Siegwart.
\newblock Fusion of {IMU} and vision for absolute scale estimation in monocular
  {SLAM}.
\newblock {\em Journal of intelligent \& robotic systems}, 61(1-4):287--299,
  2011.

\bibitem{pless2003using}
Robert Pless.
\newblock Using many cameras as one.
\newblock In {\em IEEE Conference on Computer Vision and Pattern Recognition},
  volume~2, pages II--587, 2003.

\bibitem{quan1999linear}
Long Quan and Zhongdan Lan.
\newblock Linear n-point camera pose determination.
\newblock {\em IEEE Transactions on Pattern Analysis and Machine Intelligence},
  21(8):774--780, 1999.

\bibitem{raposo2016theory}
Carolina Raposo and Joao~P Barreto.
\newblock Theory and practice of structure-from-motion using affine
  correspondences.
\newblock In {\em IEEE Conference on Computer Vision and Pattern Recognition},
  pages 5470--5478, 2016.

\bibitem{SaurerVasseur-457}
Olivier Saurer, Pascal Vasseur, R{\'e}mi Boutteau, C{\'e}dric Demonceaux, Marc
  Pollefeys, and Friedrich Fraundorfer.
\newblock Homography based egomotion estimation with a common direction.
\newblock {\em IEEE Transactions on Pattern Analysis and Machine Intelligence},
  39(2):327--341, 2016.

\bibitem{scaramuzza2011visual}
Davide Scaramuzza and Friedrich Fraundorfer.
\newblock Visual odometry: The first 30 years and fundamentals.
\newblock {\em IEEE Robotics \& Automation Magazine}, 18(4):80--92, 2011.

\bibitem{schoenberger2016sfm}
Johannes~L Sch{\"o}nberger and Jan-Michael Frahm.
\newblock Structure-from-motion revisited.
\newblock In {\em IEEE Conference on Computer Vision and Pattern Recognition},
  pages 4104--4113, 2016.

\bibitem{sturm2012benchmark}
J{\"u}rgen Sturm, Nikolas Engelhard, Felix Endres, Wolfram Burgard, and Daniel
  Cremers.
\newblock A benchmark for the evaluation of {RGB-D} {SLAM} systems.
\newblock In {\em IEEE/RSJ International Conference on Intelligent Robots and
  Systems}, pages 573--580, 2012.

\bibitem{Sweeney_2015_ISMAR}
Chris Sweeney, John Flynn, Benjamin Nuernberger, Matthew Turk, and Tobias
  H{\"o}llerer.
\newblock Efficient computation of absolute pose for gravity-aware augmented
  reality.
\newblock In {\em IEEE International Symposium on Mixed and Augmented Reality},
  pages 19--24, 2015.

\bibitem{sweeney2014solving}
Chris Sweeney, John Flynn, and Matthew Turk.
\newblock Solving for relative pose with a partially known rotation is a
  quadratic eigenvalue problem.
\newblock In {\em IEEE International Conference on 3D Vision}, pages 483--490,
  2014.

\bibitem{sweeney2015computing}
Chris Sweeney, Laurent Kneip, Tobias Hollerer, and Matthew Turk.
\newblock Computing similarity transformations from only image correspondences.
\newblock In {\em IEEE Conference on Computer Vision and Pattern Recognition},
  pages 3305--3313, 2015.

\bibitem{ventura2015efficient}
Jonathan Ventura, Clemens Arth, and Vincent Lepetit.
\newblock An efficient minimal solution for multi-camera motion.
\newblock In {\em IEEE International Conference on Computer Vision}, pages
  747--755, 2015.

\bibitem{Zhao2020GEM}
Ji Zhao, Wanting Xu, and Laurent Kneip.
\newblock A certifiably globally optimal solution to generalized essential
  matrix estimation.
\newblock In {\em IEEE Conference on Computer Vision and Pattern Recognition},
  pages 12034--12043, 2020.

\end{thebibliography}
	}
	
	\newpage
	\appendix
	\large
	\begin{center}
		{\bf Supplementary Material }
	\end{center}
	\normalsize

\section{\label{sec:6DOFmotion_supp}Geometric Constraints from ACs}
For AC $({\mathbf{x}}_{ij}, {\mathbf{x}}'_{ij}, \mathbf{A})$, we get three polynomials for six unknowns $\{q_x, q_y, q_z, t_x, t_y, t_z\}$ from Eqs.~\eqref{GECS6dof} and~\eqref{eq:E6dof_Ac6} in the paper. After separating $q_x$, $q_y$, $q_z$ from $t_x$, $t_y$, $t_z$, we arrive at equation system 	
{\begin{equation} 
	\frac{1}{1+q_x^2+q_y^2+q_z^2}\underbrace {\begin{bmatrix}
		{M_{11}}&  {M_{12}}&   {M_{13}}& {M_{14}}\\
		{M_{21}}&  {M_{22}}&   {M_{23}}& {M_{24}}\\
		{M_{31}}&  {M_{32}}&   {M_{33}}& {M_{34}}
		\end{bmatrix}}_{{\mathbf{M}}\left( {{q_x,q_y,q_z}} \right)}
	\begin{bmatrix}
	{{{t}_x}}\\
	{{{t}_y}}\\
	{{{t}_z}}\\
	1
	\end{bmatrix} = {\mathbf{0}},
	\label{eq:euq_qxqyqz1}
\end{equation}}	
where the elements $M_{ij}$ $(i=1,\ldots,3; j=1,\ldots,4)$ of the coefficient matrix ${\mathbf{M}(q_x,q_y,q_z)}$ are formed by the polynomial coefficients and three unknown variables $q_x,q_y,q_z$:
\begin{equation}
{\mathbf{M}(q_x,q_y,q_z)} = \begin{bmatrix}
[2]&\ [2]&\ [2]&\ [2]\\
[2]&\ [2]&\ [2]&\ [2]\\
[2]&\ [2]&\ [2]&\ [2]
\end{bmatrix},
\label{eq:M_qxqyqz2}
\end{equation}
where $[N]$ denotes a polynomial of degree $N$ in variables $q_x,q_y,q_z$.

Equation~\eqref{eq:euq_qxqyqz1} imposes three independent constraints on six unknowns $\{q_x, q_y, q_z, t_x, t_y, t_z\}$. This constraint can be easily generalized to special cases of multi-camera motion, \emph{e.g.}, planar motion and known vertical direction. 
\section{\label{sec:planarmotion_Supp}Relative Pose Under Planar Motion}
\subsection{Details about the Coefficient Matrix ${\mathbf{M}(q_y)}$}
Refer to Eq.~\eqref{eq:euq_q1} in the paper, three constraints obtained from a single AC are stacked into three equations in three unknowns. The elements $M_{ij}$ $(i=1,\small{\ldots},3; j=1,\small{\ldots},3)$ of the coefficient matrix ${\mathbf{M}(q_y)}$ are formed by the polynomial coefficients and one unknown variable $q_y$, which can be described as:
\begin{equation}
{\mathbf{M}(q_y)} = \begin{bmatrix}
[2]&\ [2]&\ [2]\\
[2]&\ [2]&\ [2]\\
[2]&\ [2]&\ [2]
\end{bmatrix},
\label{eq:M_qy3}
\end{equation}
where $[N]$ denotes a polynomial of degree $N$ in variable $q_y$.

\begin{table*}[t]
	\begin{center}
		\setlength{\tabcolsep}{0.9mm}{
			\scalebox{1.0}{
				\begin{tabular}{|c||c|c|c|c|c|c|c|c|c|c|}
					\hline
					\small{Methods} &  \footnotesize{17pt-Li~\cite{li2008linear}} & \footnotesize{8pt-Kneip~\cite{kneip2014efficient}} &  \footnotesize{6pt-}{\footnotesize{St.}}~\footnotesize{\cite{henrikstewenius2005solutions}}    &\footnotesize{4pt-Lee~\cite{hee2014relative}} & \footnotesize{4pt-Sw.~\cite{sweeney2014solving}}& \footnotesize{4pt-Liu~\cite{liu2017robust}}& \footnotesize{6AC-Ven.~\cite{alyousefi2020multi}}& \footnotesize{\textbf{1AC~plane}}&\footnotesize{\textbf{2AC~plane}}& \footnotesize{\textbf{2AC~vertical}}  \\
					\hline
					\small{Timings}& 43.3 & 102.0& 3275.4& 26.5& 22.2& 3.7&  38.1 &  \textbf{3.6}& \textbf{3.6} & 17.8\\
					\hline
		\end{tabular}}}
	\end{center}
	\vspace{-5pt}
	\caption{Run-time comparison of motion estimation algorithms (unit:~$\mu s$).}	
	\label{SolverTime}
\end{table*}

\subsection{Degenerate Case}
\begin{figure}[htbp]
	\begin{center}
		\includegraphics[width=0.8\linewidth]{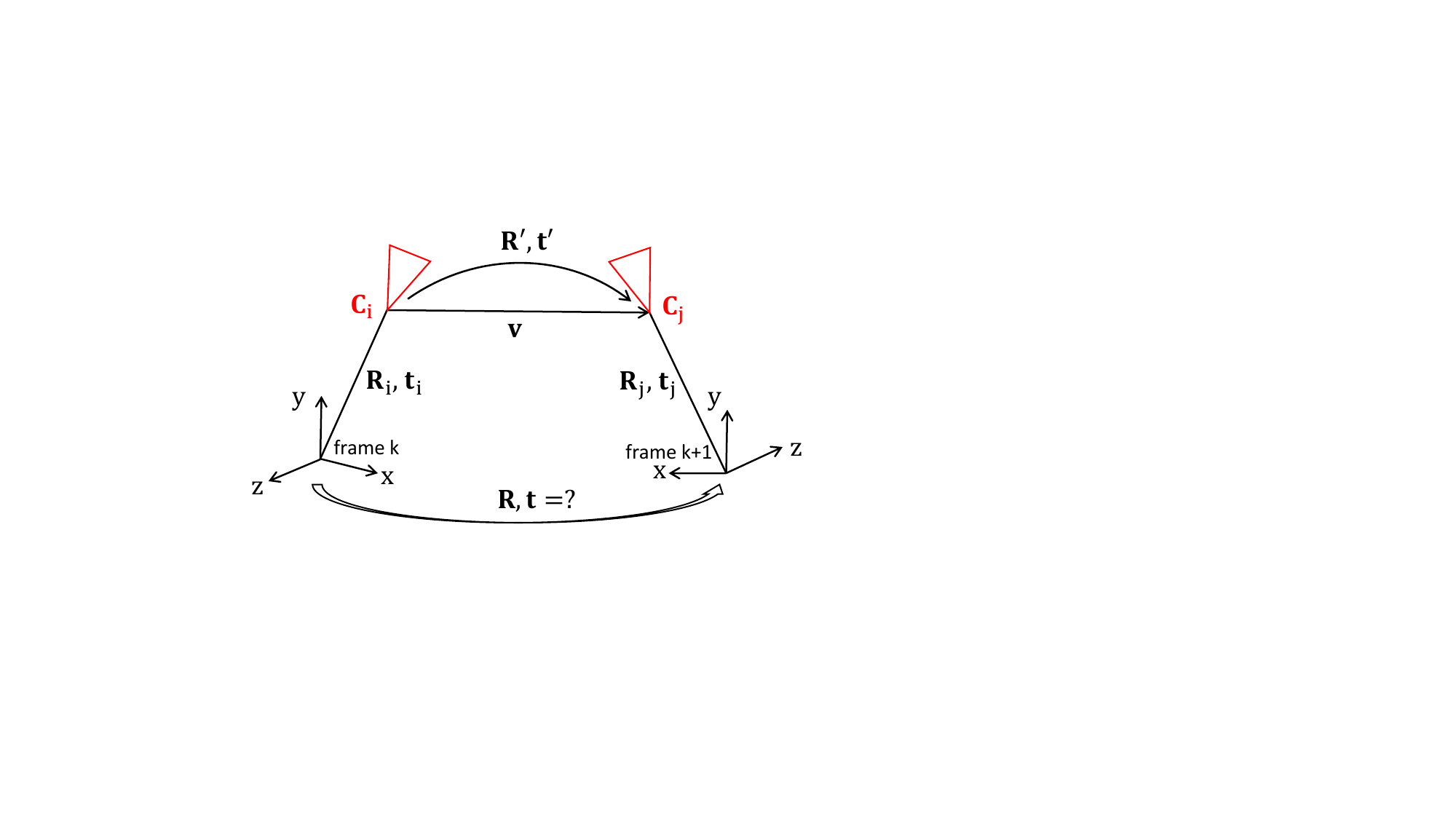}
	\end{center}
	\caption{Planar motion of a multi-camera system.}
	\label{fig:degenerate}
\end{figure}

\begin{proposition}
	\label{theorem:nister}
	Consider a multi-camera system which is under planar motion. Assume the following three conditions are satisfied. (1) The rotation axis is $y$-axis, and the translation is on $xz$-plane. (2) There is one AC across camera $C_i$ in frame $k$ and camera $C_j$ in frame $k+1$ ($C_i$ and $C_j$ can be the same or different cameras). (3) The optical centers of camera $C_i$ and $C_j$ have the same $y$-coordinate. Then this case is degenerate. Specifically, the rotation can be correctly recovered, while both the translation direction and the translation scale cannot be estimated using one AC.
\end{proposition}
\begin{proof}
	Figure~\ref{fig:degenerate} illustrates the degenerate case described in the proposition. {Note that the multi-camera reference frame is established on the multi-camera system, not on a certain camera coordinate system.} Our proof is based on the following observation: whether a case is degenerate is independent of the relative pose solvers. Based on this point, we construct a new minimal solver which is different from the proposed solver in the paper.
	
	(i) Since the multi-camera system is rotated by $y$-axis, the camera $C_i$ in frame $k$ and camera $C_j$ in frame $k+1$ are under motion with known rotation axis. Thus we can use the \texttt{1AC method}~\cite{Guan2020CVPR} for perspective cameras to estimate the relative pose between $C_i$ and $C_j$. This is a minimal solver since one AC provides 3 independent constraints and there are three unknowns (one unknown for rotation, two unknowns for translation by excluding scale-ambiguity). Denote the recovered rotation and translation between $C_i$ and $C_j$ as $(\mathbf{R}', \mathbf{t}')$, where $\mathbf{t}'$ is a unit vector. The scale of the translation vector cannot be recovered at this moment. Denote the unknown translation scale as $\lambda$.
	
	(ii) From Fig.~\ref{fig:degenerate}, we have
	{ \begin{equation}
		\begin{aligned}
		&\begin{bmatrix}
		{\mathbf{R}} & \mathbf{t}\\
		{\mathbf{0}}&{1}\\
		\end{bmatrix} = \begin{bmatrix}{\mathbf{R}_{j}}&{\mathbf{t}_{j}}\\
		{\mathbf{0}}&{1}\\
		\end{bmatrix}
		\begin{bmatrix}{\mathbf{R}'}&{ \lambda \mathbf{t}'}\\
		{\mathbf{0}}&{1}\\
		\end{bmatrix}
		\begin{bmatrix}{\mathbf{R}_{i}}&{\mathbf{t}_{i}}\\
		{\mathbf{0}}&{1}\\
		\end{bmatrix}^{-1} \\
		& \qquad \ \ =\begin{bmatrix}{{\mathbf{R}_{j}}{\mathbf{R}'}{\mathbf{R}_{i}^T}}& \ \lambda \mathbf{R}_j \mathbf{t}' + \mathbf{t}_j - \mathbf{R}_j \mathbf{R}' \mathbf{R}_i^T \mathbf{t}_i\\
		{\mathbf{0}}& \ {1}\\
		\end{bmatrix}.
		\end{aligned}
		\label{eq:trans_general}
		\end{equation}}
	
	\noindent From Eq.~\eqref{eq:trans_general}, we have
	\begin{align}
	&\mathbf{R} = \mathbf{R}_j \mathbf{R}' \mathbf{R}_i^T, \label{eq:r_equ} \\
	&\mathbf{t} = \lambda \mathbf{R}_j \mathbf{t}' + \mathbf{t}_j - \mathbf{R}_j \mathbf{R}' \mathbf{R}_i^T \mathbf{t}_i.
	\label{eq:t_equ}
	\end{align}
	From Eq.~\eqref{eq:r_equ}, the rotation $\mathbf{R}$ between frame $k$ and frame $k+1$ for the multi-camera system can be recovered.
	From Eq.~\eqref{eq:t_equ}, we have
	\begin{align}
	\lambda (\mathbf{R}_j \mathbf{t}') - \mathbf{t} + (\mathbf{t}_j - \mathbf{R} \mathbf{t}_i) = \mathbf{0}.
	\label{eq:tran_linear}
	\end{align}
	In Eq.~\eqref{eq:tran_linear}, note that $\mathbf{t} = [t_x, 0, t_z]^T$ due to planar motion. Thus this linear equation system has $3$ unknowns $\{\lambda, t_x, t_z\}$ and $3$ equations. Usually the unknowns can be uniquely determined by solving this equation system. However, if the second entry of $\mathbf{R}_j \mathbf{t}'$ is zero, it can be verified that $\lambda$ becomes a free parameter. In other words, the translation cannot be determined and this is a degenerate case.
	
	(iii) Finally, we exploit the geometric meaning of the degenerate case, i.e., the second entry of $\mathbf{R}_j \mathbf{t}'$ is zero. Denote the normalized vector originated from $C_i$ to $C_j$ as $\mathbf{v}$. Since $\mathbf{v}$ represents the normalized translation vector between $C_i$ and $C_j$, the coordinates of $\mathbf{v}$ in reference of camera $C_j$ is $\mathbf{t}'$. Further, the coordinates of $\mathbf{v}$ in frame $k+1$ is $\mathbf{R}_j \mathbf{t}'$. The second entry of $\mathbf{R}_j \mathbf{t}'$ is zero means that the endpoints of $\mathbf{v}$ have the same $y$-coordinate in frame $k+1$, which is the condition~(3) in the proposition.
\end{proof}

\section{\label{sec:knownverticaldirection_Supp}Relative Pose with Known Vertical Direction}
Refer to Eq.~\eqref{eq:euq_Ev1} in the paper, four constraints obtained from two ACs are stacked into four equations in four unknowns. The elements $\tilde{M}_{ij}$ $(i=1,\small{\ldots},4; j=1,\small{\ldots},4)$ of the coefficient matrix $\tilde{\mathbf{M}}({q_y})$ are formed by the polynomial coefficients and one unknown variable $q_y$, which can be described as:	
\begin{equation}
{\tilde{\mathbf{M}}({q_y})} = \begin{bmatrix}
[2]&\ [2]&\ [2]&\ [2]\\
[2]&\ [2]&\ [2]&\ [2]\\
[2]&\ [2]&\ [2]&\ [2]\\
[2]&\ [2]&\ [2]&\ [2]
\end{bmatrix},
\label{eq:M_qy4}
\end{equation}
where $[N]$ denotes a polynomial of degree $N$ in variable $q_y$. 

\section{\label{sec:experiments_supp}Experiments}
\subsection{Efficiency Comparison}
The runtimes of the solvers are evaluated on an Intel(R) Core(TM) i7-7800X 3.50GHz. All algorithms are implemented in C++. Methods \texttt{17pt-Li}, \texttt{8pt-Kneip} and \texttt{6pt-Stewenius} are provided in the OpenGV library~\cite{kneip2014opengv}. We implemented the \texttt{4pt-Lee} method. For methods \texttt{4pt-Sweeney}, \texttt{4pt-Liu} and \texttt{6AC-Ventura}, we used their publicly available implementations from GitHub. The average, over 10,000 runs, processing times of the solvers are shown in Table~\ref{SolverTime}. The runtimes of the methods \texttt{1AC~plane} , \texttt{2AC~plane} and \texttt{4pt-Liu} are the lowest, because these methods solve the 4-degree polynomial equation. The \texttt{2AC~vertical} which solves the 6-degree polynomial equation also requires low computation time.

\subsection{Numerical Stability}
\begin{figure}[t]
	\begin{center}
		\subfigure[]
		{
			\includegraphics[width=0.47\linewidth]{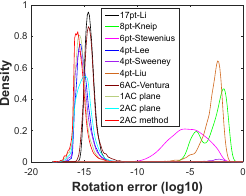}
		}
		\subfigure[]
		{
			\includegraphics[width=0.47\linewidth]{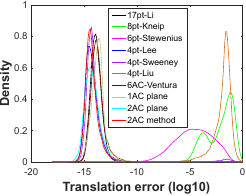}
		}
	\end{center}
	\vspace{-5pt}
	\caption{Probability density functions over estimation errors in the noise-free case (10 000 runs). The horizontal axis represents the log$_{10}$ errors and the vertical axis represents the density. (a) reports the rotation error. (b) reports the translation error. The proposed \texttt{1AC~plane} method, \texttt{2AC~plane} method and \texttt{2AC~vertical} are compared against \texttt{17pt-Li}~\cite{li2008linear}, \texttt{8pt-Kneip}~\cite{kneip2014efficient}, \texttt{6pt-Stew{\'e}nius}~\cite{henrikstewenius2005solutions}, \texttt{4pt-Lee}~\cite{hee2014relative}, \texttt{4pt-Sweeney} ~\cite{sweeney2014solving}, \texttt{4pt-Liu}~\cite{liu2017robust} and \texttt{6AC-Ventura}~\cite{alyousefi2020multi}.}
	\label{fig:Numerical}
\end{figure}    
Figure~\ref{fig:Numerical} reports the numerical stability of the solvers in the noise-free case. The procedure is repeated 10,000 times. The empirical probability density functions (vertical axis) are plotted as the function of the log$_{10}$ estimated errors (horizontal axis). Methods \texttt{1AC~plane}, \texttt{2AC~plane}, \texttt{2AC~vertical},  \texttt{17pt-Li}\cite{li2008linear}, \texttt{4pt-Lee}~\cite{hee2014relative}, \texttt{4pt-Sweeney}~\cite{sweeney2014solving} and \texttt{6AC-Ventura}~\cite{alyousefi2020multi} are numerically stable. It can also be seen that the \texttt{4pt-Sweeney} method has a small peak, both in the rotation and translation error curves, around $10^{-2}$. The \texttt{8pt-Kneip} method based on iterative optimization is susceptible to falling into local minima. Due to the use of first-order approximation of the relative rotation, the \texttt{4pt-Liu} method inevitably has greater than zero error in the noise-free case.

\subsection{{Planar Motion Estimation}}
In addition to efficiency and numerical stability, another important factor for a solver is the minimal number of required image points. The iteration number $N$ of RANSAC can be computed by $N=\log(1-p)/\log(1-(1-\epsilon)^s)$, where $s$ is the number of minimal image points, $\epsilon$ is the outlier ratio, and $p$ is the success probability. For a probability of success  $p$ = $99\%$, the RANSAC iterations needed with respect to the outlier ratio needed are shown in Figure~\ref{fig:RANSACIteration}. It can be seen that the iteration number of the RANSAC estimator increases exponentially with respect to the number of image points needed. For example, in a percentage of outliers $\epsilon$ = $50\%$, when the solvers require 1, 2, 4, 6, 8 and 17 points, the RANSAC estimator need 7, 16, 71, 292, 1177 and 603607 iterations, respectively. The proposed \texttt{1AC~plane} method which only uses a single AC requires the lowest number of RANSAC iterations. Since the proposed \texttt{2AC~plane} method need two ACs, the iteration number of RANSAC is also low in comparison to PC-based methods. Thus, our solvers can be used efficiently for detecting a correct inlier set when integrating them into the RANSAC framework. 
\begin{figure}[t]
	\begin{center}
		\includegraphics[width=0.75\linewidth]{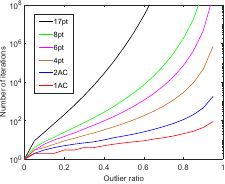}
	\end{center}
	\vspace{-5pt}
	\caption{Comparison of the RANSAC iteration number for $99\%$ of success probability.}
	\label{fig:RANSACIteration}
\end{figure} 

We evaluate the performance of the proposed \texttt{1AC~plane} method and \texttt{2AC~plane} method for outlier detection in presence of outliers. The outlier ratio is set to $50\%$. The other configurations of this synthetic experiment are set as same as using in Figures~\ref{fig:RT_planar}(d--f) in the paper. Figure~\ref{fig:Inlierset} shows the performance of the proposed methods against planar motion noise. It is interesting to see that the \texttt{1AC~plane} method recovers more than $50\%$ inliers and requires fewer number of RANSAC iterations, even though it performs poorly in translation estimation as shown in Figures~\ref{fig:RT_planar}(e--f) in the paper. Thus, the \texttt{1AC~plane} method has the advantage of detecting a correct inlier set efficiently, which can then be used for accurate motion estimation with non-linear optimization.

\begin{figure}[tbp]
	\begin{center}
		\subfigure[]
		{
			\includegraphics[width=0.47\linewidth]{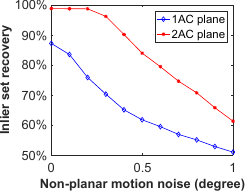}
		}
		\subfigure[]
		{
			\includegraphics[width=0.45\linewidth]{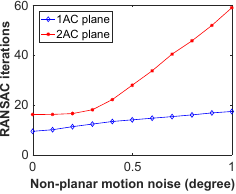}
		}
	\end{center}
	\vspace{-5pt}
	\caption{Rotation and translation error with varying planar motion noise. The image noise is fixed at $0.5$ pixel and the outlier ratio is set to $50\%$.}
	\label{fig:Inlierset}
\end{figure} 

\subsection{Motion with Known Vertical Direction}
\begin{figure}[t]
	\begin{center}
		\includegraphics[width=0.95\linewidth]{figure/Legend_2AC4DOF.pdf}\\
		\vspace{-5pt} 
		\subfigure[\scriptsize{${\varepsilon _{\bf{R}}}$ with image noise}]
		{
			\includegraphics[width=0.303\linewidth]{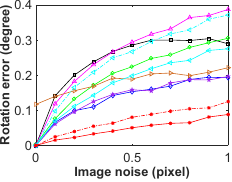}
		}
		\subfigure[\scriptsize{${\varepsilon _{\bf{R}}}$ with pitch angle noise}]
		{
			\includegraphics[width=0.303\linewidth]{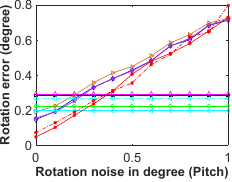}
		}
		\subfigure[\scriptsize{${\varepsilon _{\bf{R}}}$ with roll angle noise}]
		{
			\includegraphics[width=0.303\linewidth]{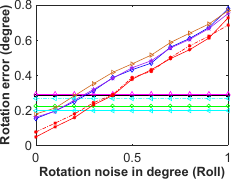}
		}
		\subfigure[\scriptsize{${\varepsilon _{\bf{t}}}$ with image noise}]
		{
			\includegraphics[width=0.303\linewidth]{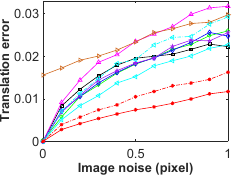}
		}
		\subfigure[\scriptsize{${\varepsilon _{\bf{t}}}$ with pitch angle noise}]
		{
			\includegraphics[width=0.303\linewidth]{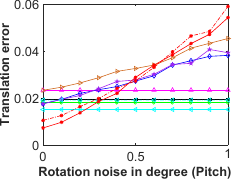}
		}
		\subfigure[\scriptsize{${\varepsilon _{\bf{t}}}$ with roll angle noise}]
		{
			\includegraphics[width=0.303\linewidth]{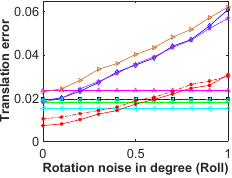}
		}
	\end{center}
	\vspace{-5pt}
	\caption{Rotation and translation error under forward motion with known vertical direction. Upper row: rotation error. Bottom row: translation error. (a,d): varying image noise. (b,e) and (c,f): varying IMU angle noise and fixed $1.0$ pixel std.\ image noise.} 
	\label{fig:RTForwardMotion_1AC}
\end{figure}

\begin{figure}[t]
	\begin{center}
		\includegraphics[width=0.95\linewidth]{figure/Legend_2AC4DOF.pdf}\\
		\subfigure[\scriptsize{${\varepsilon _{\bf{R}}}$ with image noise}]
		{
			\includegraphics[width=0.303\linewidth]{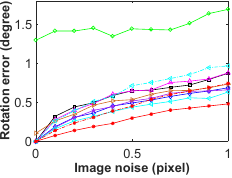}
		}
		\subfigure[\scriptsize{${\varepsilon _{\bf{R}}}$ with pitch angle noise}]
		{
			\includegraphics[width=0.303\linewidth]{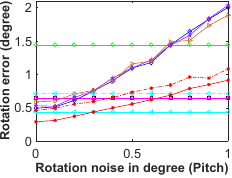}
		}
		\subfigure[\scriptsize{${\varepsilon _{\bf{R}}}$ with roll angle noise}]
		{
			\includegraphics[width=0.303\linewidth]{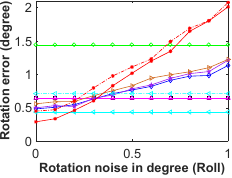}
		}
		\subfigure[\scriptsize{${\varepsilon _{\bf{t}}}$ with image noise}]
		{
			\includegraphics[width=0.303\linewidth]{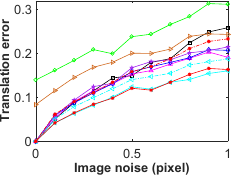}
		}
		\subfigure[\scriptsize{${\varepsilon _{\bf{t}}}$ with pitch angle noise}]
		{
			\includegraphics[width=0.303\linewidth]{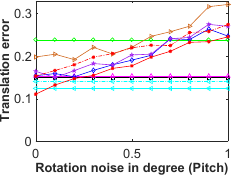}
		}
		\subfigure[\scriptsize{${\varepsilon _{\bf{t}}}$ with roll angle noise}]
		{
			\includegraphics[width=0.303\linewidth]{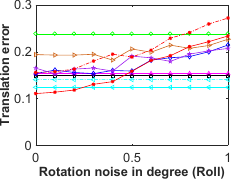}
		}
	\end{center}
	\vspace{-5pt}
	\caption{Rotation and translation error under sideways motion with known vertical direction. Upper row: rotation error. Bottom row: translation error. (a,d): varying image noise. (b,e) and (c,f): varying IMU angle noise and fixed $1.0$ pixel std.\ image noise.}
	\label{fig:RTSidewaysMotion_1AC}
\end{figure}

In this section we show the performance of the proposed \texttt{2AC~vertical} under forward and sideways motion. Figure~\ref{fig:RTForwardMotion_1AC} shows the performance of the proposed \texttt{2AC~vertical} under forward motion. It can be seen that \texttt{2AC~vertical} outperforms the comparative methods against image noise and provides comparable accuracy for increasing IMU noise, even though the size of the square is 20 pixels. Figure~\ref{fig:RTSidewaysMotion_1AC} shows the performance of the proposed \texttt{2AC~vertical} under sideways motion. The results demonstrate that when the side length of the square is 40 pixels, the \texttt{2AC~vertical} performs basically better than all compared methods against image noise and achieves comparable performance for increasing noise on the IMU data.

\subsection{{Using PCs converted from ACs}}
In this set of experiments, we evaluate the performance of PC-based solvers using the PCs converted from ACs. Given an AC as $({\mathbf{x}}, {\mathbf{x}}', \mathbf{A})$, where ${\mathbf{x}}$ and ${\mathbf{x}}'$ are the image coordinates of feature point in two views and $\mathbf{A}$ is the corresponding 2$\times$2 local affine transformation. Three generated PCs include an image point pair of AC and two hallucinated image point pairs calculated by the local affine transformation. Since local affine transformations are defined as the partial derivative, w.r.t. the image directions, of the related homography, they are valid only infinitesimally close to the image coordinates of AC. Thereby, one AC can only provide three approximate PCs – the error is not zero even for noise-free input~\cite{barath2018efficient}. Three approximate PCs converted from one AC can be computed as follows~\cite{barath2020making}: ${\mathbf{x}}+[0, \ w, \ 0; \ 0, \ 0, \ w]$ and ${\mathbf{x}}'+\mathbf{A}[0, \ w, \ 0; \ 0, \ 0, \ w]$, where $w$ determines the distribution area of the generated PCs. To evaluate the performance of PC-based solvers with different distribution area, $w$ is set to 1, 5 and 10 pixels, respectively. 

\begin{figure}[t]
	\begin{center}
		\includegraphics[width=0.8\linewidth]{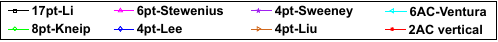}\\
		\subfigure[\scriptsize{${\varepsilon _{\bf{R}}}$ with image noise}]
		{
			\includegraphics[width=1.0\linewidth]{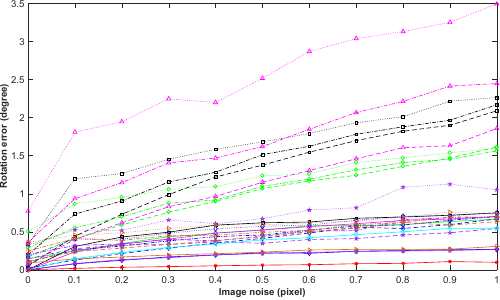}
		}\\
		\subfigure[\scriptsize{${\varepsilon _{\bf{t}}}$ with image noise}]
		{
			\includegraphics[width=1.0\linewidth]{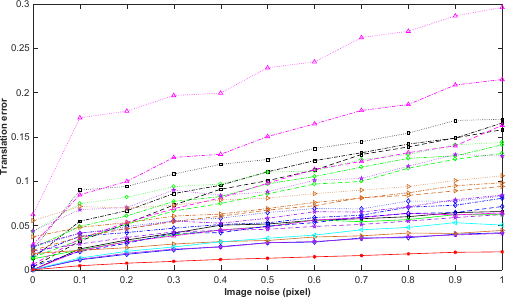}
		}
	\end{center}
	\vspace{-5pt}
	\caption{Rotation and translation error with varying image noise under random motion with known vertical direction. Solid line indicates using image point pairs of ACs. Dashed line, dash-dotted line and dotted line indicate using the hallucinated PCs, which are generated with different distribution area $w$ = 1, 5, 10 pixels, respectively.}
	\label{fig:RTRandomMotion_1ACto3PCs}
\end{figure}

\begin{table*}[t]
	\begin{center}
		\setlength{\tabcolsep}{0.9mm}{
			\scalebox{0.85}{
				\begin{tabular}{|c||c|c|c|c|c|c|c|c|c|}
					\hline
					\multirow{2}{*}{\small{Part}} &  \footnotesize{17pt-Li~\cite{li2008linear}} & \footnotesize{8pt-Kneip~\cite{kneip2014efficient}}  &  \footnotesize{6pt-}{\footnotesize{St.}}~\footnotesize{\cite{henrikstewenius2005solutions}}   &  \footnotesize{4pt-Lee~\cite{hee2014relative}} & \footnotesize{4pt-Sw.~\cite{sweeney2014solving}}& \footnotesize{4pt-Liu~\cite{liu2017robust}}& \footnotesize{6AC-Ven.~\cite{alyousefi2020multi}}& \footnotesize{\textbf{2AC~plane}}& \footnotesize{\textbf{2AC~vertical}} \\
					\cline{2-10}
					& ${\varepsilon _{\bf{R}}}$\qquad\ ${\varepsilon _{\bf{t}}}$      &  ${\varepsilon _{\bf{R}}}$\qquad\ ${\varepsilon _{\bf{t}}}$      &   ${\varepsilon _{\bf{R}}}$\qquad\ ${\varepsilon _{\bf{t}}}$     &   ${\varepsilon _{\bf{R}}}$\qquad\ ${\varepsilon _{\bf{t}}}$  &   ${\varepsilon _{\bf{R}}}$\qquad\ ${\varepsilon _{\bf{t}}}$   &   ${\varepsilon _{\bf{R}}}$\qquad\ ${\varepsilon _{\bf{t}}}$ &   ${\varepsilon _{\bf{R}}}$\qquad\ ${\varepsilon _{\bf{t}}}$  &   ${\varepsilon _{\bf{R}}}$\qquad\ ${\varepsilon _{\bf{t}}}$  &   ${\varepsilon _{\bf{R}}}$\qquad\ ${\varepsilon _{\bf{t}}}$\\
					\hline
					\small{01 (3376 images)}&  0.161 \ 2.680 &  0.156 \ 2.407 & 0.203 \  2.764 & 0.083 \ 1.780 & 0.078 \ 1.659  & 0.108 \ 1.941 &  0.143 \ 2.366 & 0.344 \ 2.284 & \textbf{0.057} \ \textbf{1.469}\\
					\hline					
		\end{tabular}}}
	\end{center}
	\caption{Rotation and translation error on \texttt{nuScenes} sequences (unit: degree).}
	\label{VerticalRTErrrornuScenes}
\end{table*}

Take relative pose estimation with known vertical direction for an example. A total of 1000 trials are carried out in the synthetic experiment. In each test, 100 ACs are generated randomly with 40*40 support region. In the RANSAC loop, six ACs and two ACs are selected randomly for the \texttt{6AC-Ventura} method and the proposed \texttt{2AC~vertical} method, respectively. The hallucinated PCs converted from a minimal number of ACs are used as input for the PC-based solvers. Thus, 6, 3 and 2 ACs are selected randomly for the \texttt{17pt-Li} solver~\cite{li2008linear}, the \texttt{8pt-Kneip} solver~\cite{kneip2014efficient}, and the solvers \texttt{6pt-Stew{\'e}nius}~\cite{henrikstewenius2005solutions}, \texttt{4pt-Lee} ~\cite{hee2014relative}, \texttt{4pt-Sweeney}~\cite{sweeney2014solving} and \texttt{4pt-Liu}~\cite{liu2017robust}, respectively. Note that the hallucinated PCs converted from  ACs are only used for hypothesis generation, and the inlier set is found by evaluating the image point pairs of ACs. The solution which produces the highest number of inliers is chosen. The other configurations of this synthetic experiment are set as same as using in Figures~\ref{fig:RT_1AC}(a) and (d) in the paper. 

Figure~\ref{fig:RTRandomMotion_1ACto3PCs} shows the performance of the PC-based solvers against image noise in the random motion case. The estimation results using the image point pairs of ACs are represented by solid lines. The estimation results using the hallucinated PCs generated with different distribution area are represented by dashed line ($w$ = 1 pixel), dash-dotted line ($w$ = 5 pixels) and dotted line ($w$ = 10 pixels), respectively. We have the following observations. (1) The PC-based solvers using the hallucinated PCs perform worse than using the image point pairs of AC. Because the conversion error between each AC and three PCs is newly introduced. It can be seen that the estimation error of PC-based solvers using the hallucinated PCs is not zero even for image noise-free input. Moreover, the hallucinated PCs generated by each AC are near each other which may be a degenerate case for the PC-based solvers. (2) The performance of PC-based solvers is influenced by the different distribution area of hallucinated PCs. Since a smaller distribution area causes smaller conversion error between ACs and PCs, the PC-based solvers have better performance with smaller distribution area. (3) The performance of the proposed \texttt{2AC vertical} method is best. Because the AC-based solvers use the relationship between local affine transformations and epipolar lines (Eq.~\eqref{eq:E6dof_Ac6} in the paper). This is a strictly satisfied constraint and does not result in any error for noise-free input. In addition, the \texttt{2AC vertical} method is robust to image noise and performs better than the \texttt{6AC-Ventura} method.

\subsection{Experiments on KITTI dataset}
We also show the empirical cumulative error distributions for KITTI sequence 00. These values are calculated from the same values which were used for creating Table~\ref{VerticalRTErrror} in the paper. Figure~\ref{fig:RTCDF} shows the proposed \texttt{2AC~vertical} offers the best overall performance in comparison to state-of-the-art methods.
\begin{figure}[htbp]
	\begin{center}
		\subfigure[]
		{
			\includegraphics[width=0.473\linewidth]{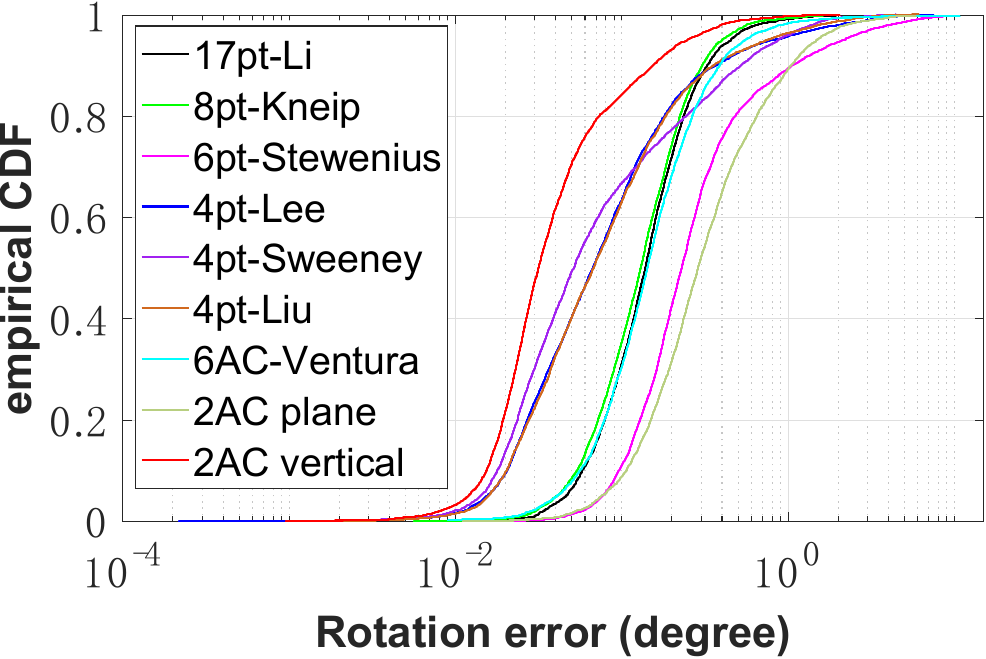}
		}
		\subfigure[]
		{
			\includegraphics[width=0.473\linewidth]{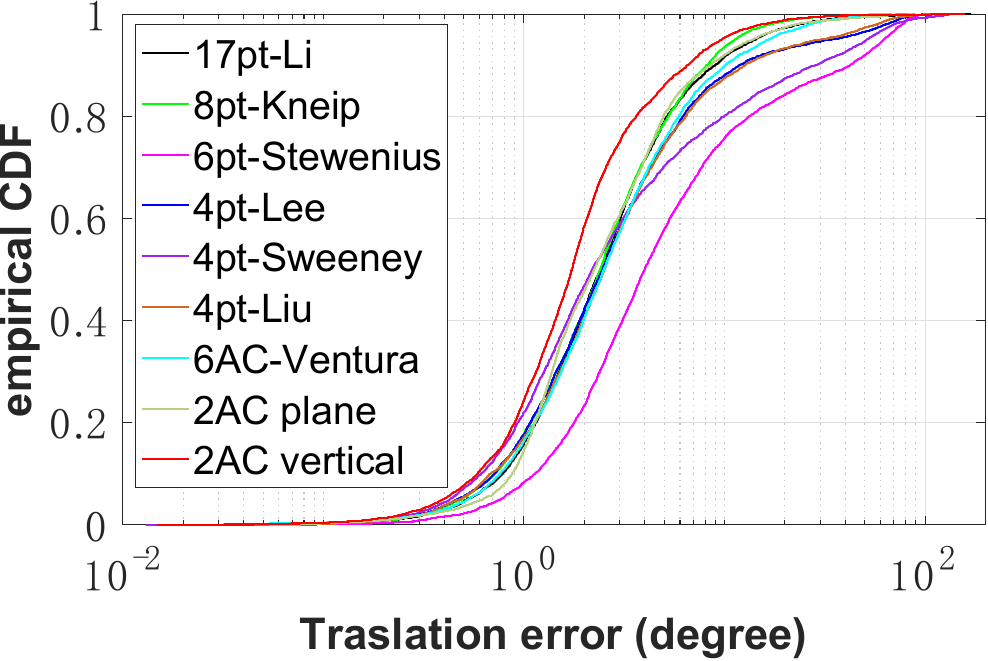}
		}
	\end{center}
	\vspace{-5pt}
	\caption{Empirical cumulative error distributions for KITTI sequence 00. (a) reports the rotation error. (b) reports the translation error. The proposed \texttt{2AC~plane} method and \texttt{2AC~vertical} are compared against \texttt{17pt-Li}~\cite{li2008linear}, \texttt{8pt-Kneip}~\cite{kneip2014efficient}, \texttt{6pt-Stew{\'e}nius}~\cite{henrikstewenius2005solutions},  \texttt{4pt-Lee}~\cite{hee2014relative}, \texttt{4pt-Sweeney} ~\cite{sweeney2014solving} and \texttt{4pt-Liu}~\cite{liu2017robust}.}
	\label{fig:RTCDF}
\end{figure}

\begin{figure*}[h]
	\begin{center}
		\subfigure[\texttt{8pt-Kneip}]
		{
			\includegraphics[width=0.30\linewidth]{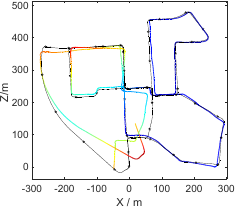}
		}
		\subfigure[\texttt{4pt-Sweeney}]
		{
			\includegraphics[width=0.30\linewidth]{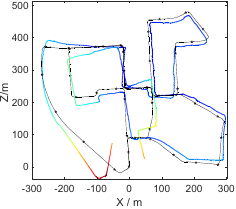}
		}
		\subfigure[\texttt{2AC~vertical}]
		{
			\includegraphics[width=0.355\linewidth]{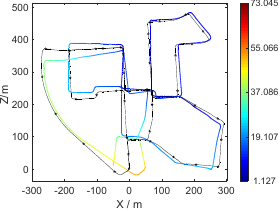}
		}
	\end{center}
	\caption{Estimated trajectories without any post-refinement. The relative pose measurements between consecutive frames are directly concatenated. The colorful curves are the trajectories estimated by \texttt{8pt-Kneip}~\cite{kneip2014efficient}, \texttt{4pt-Sweeney}~\cite{sweeney2014solving} and \texttt{2AC~vertical}. Black curves with stars are the ground truth trajectories. Best viewed in color.}
	\label{fig:trajectory}
\end{figure*}

To visualize the comparison results, the estimated trajectory for sequence 00 is plotted in Fig.~\ref{fig:trajectory}. We are directly concatenating frame-to-frame relative pose measurements without any post-refinement. The trajectory for \texttt{2AC~vertical} is compared with the two best performing comparison methods in sequence 00 based on Table~\ref{VerticalRTErrror} in the paper: \texttt{8pt-Kneip} in 6DOF motion case and \texttt{4pt-Sweeney} in 4DOF motion case. Since all methods were not able to estimate the scale correctly, in particular for the many straight parts of the trajectory, the ground truth scale is used to plot the trajectories. Then the trajectories are aligned with the ground truth and the color along the trajectory encodes the absolute trajectory error (ATE)~\cite{sturm2012benchmark}. Even though all trajectories have a significant accumulation of drift, it can still be seen that the \texttt{2AC~vertical} method has the smallest ATE among the compared trajectories.   

\subsection{Experiments on nuScenes dataset}
We also test the performance of our methods on the \texttt{nuScenes} dataset~\cite{Caesar_2020_CVPR}, which consists of consecutive keyframes from 6 cameras. All the keyframes of Part 1 are used for the evaluation and there are 3376 images in total. The ground truth pose is provided from a lidar map-based localization scheme. Similar to the experiments on \texttt{KITTI} dataset, the ACs between consecutive keyframes in each camera are established by applying the ASIFT~\cite{morel2009asift} detector. All solvers are used within RANSAC.

Table~\ref{VerticalRTErrrornuScenes} shows the results of the rotation and translation estimation for the Part1 of \texttt{nuScenes} dataset. The median error is used to evaluate the estimation accuracy. It can be seen that the proposed \texttt{2AC vertical} method offers the best overall performance among all the methods.

\end{document}